\newif\ifsingle

\newif\ifFullVersion
\FullVersiontrue 

\ifsingle
\documentclass[11pt,draftclsnofoot, onecolumn]{IEEEtran}		
\else		
\documentclass[10pt,final, twocolumn]{IEEEtran}
\fi

\usepackage{times}
\usepackage{amsmath,dsfont}
\usepackage{amssymb,amsthm}
\usepackage{epsfig,verbatim}
\usepackage{subfigure}
\usepackage{setspace}
\usepackage{color}
\usepackage{cite}
\usepackage{epstopdf}
\usepackage{graphics}
\usepackage{accents}
\usepackage{acronym}
\usepackage{hyperref}
\hypersetup{bookmarks,colorlinks}
\usepackage{booktabs}
\usepackage{mathtools}
\usepackage{steinmetz}
\usepackage{enumitem}
\usepackage{soul}

\usepackage[ruled,linesnumbered]{algorithm2e}  
\SetKwInput{KwData}{\textbf{Init}} 
\usepackage{geometry}
\DeclareMathOperator*{\argmin}{\arg\min} 
\newcommand{\myVec}[1]{{\boldsymbol{#1}}}
\newcommand{\myMat}[1]{{\boldsymbol{#1}}}
\newcommand{\mySet}[1]{\mathcal{#1}}

\newcommand{\myWeights}{\thetabf}
\newcommand{\Opt}{^{*}}
\newcommand{\SGDIter}{H}
\newcommand{\Objective}{F}
\newcommand{\SmoothParam}{L}
\newcommand{\ConvParam}{\mu}
\newcommand{\StepSize}{\eta}

\newcommand{\E}{\mathds{E}}		 			

\newcommand{\StepSizeNume}{\rho}

\newcommand{\thetabf}{\boldsymbol{\theta}}   
\newcommand{\thetabar}{\boldsymbol{\bar{\thetabf}}} 
\newcommand{\tstar}{\boldsymbol{\theta}^\star}

 \ifFullVersion
  \newcommand{\YonRep}[1]{}
  \newcommand{\eqspace}{}
 \else
  \newcommand{\YonRep}[1]{}
  \newcommand{\eqspace}{\vspace{-0.1cm}}
 \fi

\newtheorem{theorem}{Theorem}
\newtheorem{corollary}{Corollary}

\newtheorem{lemma}{Lemma}

\definecolor{NewColor}{rgb}{0,0,0} 

\ifsingle
\newcommand{\figWidth}{0.65\columnwidth}

\newcommand{\figHeight}{0.25\textheight}

\else
\newcommand{\figWidth}{0.85\columnwidth}

\newcommand{\figHeight}{0.19\textheight}
\fi 

\acrodef{adc}[ADC]{analog-to-digital convertor}
\acrodef{cs}[CS]{compressed sensing}
\acrodef{dtft}[DTFT]{discrete-time Fourier transform}
\acrodef{dnn}[DNN]{deep neural network} 
\acrodef{csi}[CSI]{channel state information}
\acrodef{map}[MAP]{maximum a-posteriori probability}
\acrodef{snr}[SNR]{signal-to-noise ratio}
\acrodef{bs}[BS]{base station} 
\acrodef{iot}[IOT]{Interent of Things}
\acrodef{mimo}[MIMO]{multiple-input multiple-output}
\acrodef{mse}[MSE]{mean-squared error}
\acrodef{pdf}[PDF]{probability density function}
\acrodef{rv}[RV]{random variable}
\acrodef{fec}[FEC]{forward error correction}
\acrodef{rs}[RS]{Reed-Solomon}
\acrodef{lti}[LTI]{linear time-invariant}
\acrodef{wss}[WSS]{wide-sense stationary}
\acrodef{psd}[PSD]{power spectral density}
\acrodef{ser}[SER]{symbol error rate} 
\acrodef{ber}[BER]{bit error rate} 
\acrodef{sgd}[SGD]{stochastic gradient descent} 
\acrodef{isi}[ISI]{intersymbol interference}  
\acrodef{awgn}[AWGN]{additive white Gaussian noise} 
\acrodef{ut}[UT]{user terminal} 
\acrodef{mmw}[mmWave]{millimeter wave}
\acrodef{noma}[NOMA]{non-orthognal multiple access}
\acrodef{mac}[MAC]{mulitple access channel}
\acrodef{fl}[FL]{federated learning}
\acrodef{fdm}[FDM]{frequency division multiplexing}
\acrodef{tdm}[TDM]{time division multiplexing}
\acrodef{ota}[OTA]{over-the-air}

\acrodef{cotaf}[COTAF]{convergent \ac{ota} \ac{fl}}
\acrodef{sgd}[SGD]{stochastic gradient descent}
\IEEEoverridecommandlockouts

\title{Over-the-Air Federated Learning from Heterogeneous Data
}
\author{
	\IEEEauthorblockN{Tomer Sery, Nir Shlezinger, Kobi Cohen, and Yonina C. Eldar \\
	} 
	\thanks{
	A short version of this paper that introduces the algorithm for i.i.d. data and preliminary simulation results was accepted for presentation in the 2020 IEEE Global Communications Conference (GLOBECOM) \cite{sery2020cotaf2}. 
	This research was supported by the Benoziyo Endowment Fund for the Advancement of Science, the	Estate of Olga Klein – Astrachan, the European Union’s Horizon 2020 research and innovation program under grant No. 646804-ERC-COG-BNYQ, the Israel Science Foundation under grant No. 0100101, the Israel Science Foundation under grant No. 2640/20, and the U.S.-Israel Binational Science Foundation (BSF) under grant No. 2017723.
	T. Sery, N. Shlezinger, and K. Cohen are with the School of Electrical and Computer Engineering, Ben-Gurion University of the Negev, Beer-Sheva, Israel (e-mail:seryt@post.bgu.ac.il; \{nirshl, yakovsec\}@bgu.ac.il).
	 Y. C. Eldar is with the Math and CS Faculty, Weizmann Institute of Science, Rehovot, Israel (e-mail:  yonina.eldar@weizmann.ac.il). 	
	} 
	\vspace{-1.0cm}
}
\vspace{-0.75cm}

\begin{document}
	
	\maketitle
	\pagestyle{plain}
	\thispagestyle{plain}
	\begin{abstract}
Federated learning (FL) is a framework for distributed learning of centralized models. In FL, a set of edge devices train a model using their local data, while repeatedly exchanging their trained updates with a central server. This procedure allows tuning a centralized model in a distributed fashion without having the users share their possibly private data. In this paper, we focus on over-the-air (OTA) FL, which has been suggested recently to reduce the communication overhead of FL due to the repeated transmissions of the model updates by a large number of users over the wireless channel. In OTA FL, all users simultaneously transmit their updates as analog signals over a multiple access channel, and the server receives a superposition of the analog transmitted signals. However, this approach results in the channel noise directly affecting the optimization procedure, which may degrade the accuracy of the trained model. We develop a Convergent OTA FL (COTAF) algorithm which enhances the common local stochastic gradient descent (SGD) FL algorithm, introducing precoding at the users and scaling at the server, which gradually mitigates the effect of the noise. We analyze the convergence of COTAF to the loss minimizing model and quantify the effect of a statistically heterogeneous setup, i.e. when the training data of each user obeys a different distribution. Our analysis reveals the ability of COTAF to achieve a convergence rate similar to that achievable over error-free channels. Our simulations demonstrate the improved convergence of COTAF over vanilla OTA local SGD for training using non-synthetic datasets. Furthermore, we numerically show that the precoding induced by COTAF notably improves the convergence rate and the accuracy of models trained via OTA FL.  
%
	\end{abstract}
	\vspace{-0.4cm}
	\section{Introduction} \label{sec: introduction}
	\vspace{-0.1cm} 
	Recent years have witnessed unprecedented success of machine learning methods in a broad range of applications \cite{lecun2015deep}. These systems utilize highly parameterized models, such as \acp{dnn}, trained using massive data sets. In many applications, samples are available at remote users, e.g. smartphones, and the common strategy is to gather these samples at a computationally powerful server, where the model is trained \cite{chen2019deep}. 
	Often, data sets contain private information, and thus the user may not be willing to share them with the server. Furthermore, sharing massive data sets can result in a substantial burden on the communication links between the users and the server. To allow centralized training without data sharing, \ac{fl}  was proposed in \cite{mcmahan2016communication} as a method combining distributed training with central aggregation, and is the focus of growing research attention \cite{kairouz2019advances}. \ac{fl} exploits the increased computational capabilities of modern edge devices to train a model on the users' side,  having the server periodically synchronize these local models into a global one. 

    Two of the main challenges associated with \ac{fl} are the heterogeneous nature of the data and the communication overhead induced by its training procedure \cite{kairouz2019advances}. Statistical heterogeneity arises when the data generating distributions vary between different sets of users \cite{smith2017federated}. This is typically the case in \ac{fl}, as the data available at each user device is likely to be personalized towards the specific user. When training several instances of a model on multiple edge devices using heterogeneous data, each instance can be adapted to operate under a different statistical relationship, which may limit the inference accuracy of the global model \cite{smith2017federated, shlezinger2020clustered, li2019federated}.
    
    The communication load of \ac{fl} stems from the need to repeatedly convey a massive amount of model parameters between the server and a large number of users over wireless channels \cite{li2019federated}. This is particularly relevant in uplink communications, which are typically more limited as compared to their downlink counterparts \cite{speedtest2019}. A common strategy to tackle this challenge is to reduce the amount of data exchanges between the users and the server, either by reducing the number of participating users \cite{chen2019joint, li2019convergence}, or by compressing the model parameters via quantization \cite{alistarh2017qsgd,shlezinger2020uveqfed} or sparsification \cite{aji2017sparse,alistarh2018convergence}. All these methods treat the wireless channel as a set of independent error-free bit-limited links between the users and the server. As wireless channels are shared and noisy \cite{goldsmith2005wireless}, a common way to achieve such communications is to divide the channel resources among users, e.g., by using \ac{fdm}, and have the users utilize channel codes to overcome the noise. This, however, results in each user being assigned a dedicated band whose width decreases with the number of users, which in turn increases the energy consumption required to meet a desirable communication rate and decreases the overall throughput and training speed. 

	An alternative \ac{fl} approach is to allow the users to simultaneously utilize the complete temporal and spectral resources of the uplink channel in a non-orthogonal manner. In this method, referred to as \ac{ota} \ac{fl} \cite{amiri2020machine, amiri2019federated,sery2019analog,yang2020federated, guo2020analog}, the users transmit their model updates via analog signalling, i.e., without converting to discrete coded symbols which should be decoded at the server side. Such \ac{fl} schemes exploit the inherent aggregation carried out by the shared channel as a form of \ac{ota} computation \cite{abari2016over}. 
	This strategy builds upon the fact that when the participating users operate over the same wireless network, uplink transmissions are carried out over a \ac{mac}. Model-dependent inference over \acp{mac} is relatively well-studied in  the sensor network literature, where methods for model-dependent inference over \acp{mac} and theoretical performance guarantees have been established  under a wide class of problem settings 
\ifFullVersion	
	\cite{mergen2006type, Mergen_Asymptotic_2007, Liu_Type_2007, Marano_Likelihood_2007, anandkumar2007type, cohen2013performance, nevat2014distributed, zhang2016event, cohen2018spectrum, cohen2019time}. 
\else
    \cite{ Liu_Type_2007,  cohen2013performance,  cohen2018spectrum, cohen2019time}. 
\fi
	These studies focused on model-based inference, and not on machine learning paradigms. In the context of \ac{fl}, which is a distributed machine learning setup, with \ac{ota} computations, the works \cite{amiri2020machine,amiri2019federated} considered scenarios where the model updates are sparse with an identical sparsity pattern, which is not likely to hold when the data is heterogeneous. Additional related recent works on \ac{ota} \ac{fl}, including \cite{sery2019analog,guo2020analog, seif2020wireless,liu2020privacy}, considered the distributed application of full gradient descent optimization  over noisy channels. While distributed learning based on full gradient descent admits a simplified and analytically tractable analysis, it is also less communication and computation efficient compared to local \ac{sgd}, which is the dominant optimization scheme used in \ac{fl} \cite{mcmahan2016communication, kairouz2019advances}. 
	Consequently, the \ac{ota} \ac{fl} schemes proposed in these previous works and the corresponding analysis of their convergence may not reflect the common application of \ac{fl} systems, i.e., distributed training with heterogeneous data via local \ac{sgd}.

	The main advantage of \ac{ota} \ac{fl} is that it enables the users to transmit at increased throughput, being allowed to utilize the complete available bandwidth regardless of the number of participating users. However, a major drawback of such uncoded analog signalling is that the noise induced by the channel is not handled by channel coding and thus affects the training procedure. In particular, the accuracy 
	of learning algorithms such as \ac{sgd} is known to be sensitive to noisy observations, as in the presence of noise the model can only be shown to converge to some environment of the optimal solution \cite{cesa2011online}. Combining the sensitivity to noisy observations with the limited accuracy due to statistical heterogeneity of \ac{fl} systems, implies that conventional \ac{fl} algorithms, such as local \ac{sgd} \cite{stich2018local}, exhibit degraded performance when combined with \ac{ota} computations, and are unable to converge to the optimum. This motivates the design and analysis of an \ac{fl} scheme for wireless channels that exploit the high throughput of \ac{ota} computations, while preserving the convergence properties of conventional \ac{fl} methods designed for noise-free channels.

Here, we propose the \ac{cotaf} algorithm which introduces precoding and scaling laws. \ac{cotaf} facilitates high throughput  \ac{fl} over wireless channels, while preserving the accuracy and convergence properties of the common local \ac{sgd} method for distributed learning. Being an \ac{ota} \ac{fl} scheme, \ac{cotaf} overcomes the need to divide the channel resources among the users by allowing the users to simultaneously share the uplink channel, while aggregating the global model via \ac{ota} computations. To guarantee convergence to an accurate parameter model, we introduce time-varying precoding to the transmitted signals, which accounts for the fact that the expected difference in each set of \ac{sgd} iterations is expected to gradually decrease over time. Building upon this insight, \ac{cotaf} scales the model updates by their maximal expected norm, along with a corresponding aggregation mapping at the server side, which jointly results in an equivalent model where the effect of the noise induced by the channel is mitigated over time.

	We theoretically analyze the convergence of machine learning models trained by \ac{cotaf} to the minimal achievable loss function in the presence of heterogeneous data. Our theoretical analysis focuses on scenarios in which the objective function is strongly convex and smooth, and the stochastic gradients have bounded variance. Under such scenarios, which are commonly utilized in \ac{fl} convergence studies over error-free channels \cite{stich2018local,li2019convergence,Stich2018sparsified}, noise degrades the ability to converge to the global optimum.  
	
	We provide three convergence bounds: The first  characterizes the distance between a weighted average of past models trained in a federated manner  \cite{stich2018local}; The second treats the convergence of the instantaneous model available at the end of the \ac{fl} procedure  \cite{li2019convergence}. The first two bounds consider \ac{fl} over non-fading channels. We then extend \ac{cotaf} for fading channels and characterize the corresponding convergence of the instantaneous model.  Our analysis proves that when applying \ac{cotaf}, the usage of analog transmissions over shared noisy channels does not affect the asymptotic convergence rate of local \ac{sgd} compared to \ac{fl} over error-free separate channels, while allowing the users to communicate at high throughput by avoiding the need to divide the channel resources. 
	Our convergence bounds show that the distance to the desired model is smaller when the data is closer to being i.i.d., as in \ac{fl} over error-free channels with heterogeneous data \cite{li2019convergence}.
	Unlike previous convergence proofs of \ac{ota} \ac{fl}, our analysis of \ac{cotaf} is not restricted to sparsified updates as in \cite{yang2020federated} or to full gradient descent optimization as in \cite{guo2020analog}, and holds for the typical \ac{fl} setting with \ac{sgd}-based training and heterogeneous data. 

    We evaluate \ac{cotaf} in two scenarios involving non-synthetic data sets: First, we train a linear estimator, for which the objective function is strongly convex, with the Million Song Dataset \cite{Bertin-Mahieux2011}. In such settings we demonstrate that \ac{cotaf} achieves accuracy within a minor gap from that of noise-free local \ac{sgd}, while notably outperforming \ac{ota} \ac{fl} strategies without time-varying precoding designed to facilitate convergence. 
	Then, we train a convolutional neural network (CNN) over the CIFAR-10 dataset, 
	representing a deep \ac{fl} setup with a non-convex objective, for which a minor level of noise is known to contribute to convergence as means of avoiding local mimimas \cite{Guozhong1995NoiseBackprop}. We demonstrate that \ac{cotaf} improves the accuracy of trained models when using both i.i.d and heterogeneous data. Here, \ac{cotaf} benefits from the presence of the gradually mitigated noise to achieve improved performance not only over conventional  \ac{ota} \ac{fl}, but also over noise-free local SGD.

    The rest of this paper is organized as follows: Section~\ref{sec:Model} briefly reviews the local \ac{sgd} algorithm and presents the system model of \ac{ota} \ac{fl}. Section~\ref{sec:Protocol} presents the \ac{cotaf} scheme along with its theoretical convergence analysis. Numerical results are detailed in Section~\ref{sec:sims}. Finally, Section~\ref{sec:Conclusions} provides concluding remarks.
    Detailed proofs of our main results are given in the appendix. 

	Throughout the paper, we use boldface lower-case letters for vectors, e.g., ${\myVec{x}}$. 
		The $\ell_2$ norm, stochastic expectation, and Gaussian distribution are denoted by $\| \cdot \|$, $\E[\cdot]$, and $\mathcal{N}(\cdot, \cdot)$ respectively.
	Finally, $\myMat{I}_n$ is the $n \times n$ identity matrix, and $\mathbb{R}$ is the set of real numbers.

	\vspace{-0.2cm}
	\section{System Model}
	\label{sec:Model}
	\vspace{-0.1cm}
	In this section we detail the system model for which \ac{cotaf} is derived in the following section. We first formulate the  objective of \ac{fl} in Subsection \ref{subsec:ModelOptimization}. Then, Subsection \ref{subsec:ModelComm} presents the  communication channel model over which \ac{fl} is carried out. We briefly discuss the local \ac{sgd} method, which is the common \ac{fl} algorithm, in Subsection \ref{subsec:ModelFL}, and formulate the problem in Subsection \ref{subsec:ModelProblem}.
	
	\vspace{-0.2cm}
	\subsection{Federated Learning}
	\label{subsec:ModelOptimization}
	\vspace{-0.1cm}  
		We consider a central server which trains a model consisting of $d$ parameters, represented by the vector $\thetabf \in\Theta\subset\mathbb{R}^d$, using data available at $N$ users, indexed by the set $\mathcal{N} = \{1,2,...,N\}$, as illustrated in Fig. \ref{fig:Setup1}.  
		Each user of index $n \in \mathcal{N}$ has access to a data set of $D_n$ entities, denoted by $\{\myVec{s}_i^n\}_{i=1}^{D_n}$, sampled in an i.i.d. fashion from a local distribution $\mathcal{X}_n$. The users can communicate with the central server over a wireless channel formulated in Subsection \ref{subsec:ModelComm}, but are not allowed to share their data  with the server. 
		
	To define the learning objective, we use $l(\cdot, \thetabf)$ to denote the loss function of a model parameterized by $\thetabf$. The empirical loss of the $n$th user is defined by 
 \eqspace
		\begin{equation}
		 f_n(\thetabf) \triangleq \frac{1}{D_n}\sum_{i=1}^{D_n}l(\myVec{s}_i^n; \thetabf).   
		 \label{eqn:Local cost}
		 \eqspace
		\end{equation} 
\ifFullVersion		
		The objective is the average global loss, given by
		\begin{equation}
		    F(\thetabf) \triangleq\frac{1}{N} \sum_{n=1}^N f_n(\thetabf).
		\end{equation}
		Therefore, \ac{fl} aims at recovering  
	\begin{equation}
	    \tstar \triangleq \argmin_{\thetabf \in \Theta} F(\thetabf).
	    \label{eqn:Global cost}
	\end{equation} 
\else
		\ac{fl} aims at  minimizing the average loss:
		\eqspace
	\begin{equation}
	    \tstar \triangleq \argmin_{\thetabf \in \Theta} F(\thetabf), \quad  F(\thetabf) \triangleq\frac{1}{N} \sum_{n=1}^N f_n(\thetabf).
	    \label{eqn:Global cost}
	    \eqspace
	\end{equation} 
\fi

	\begin{figure*}
	\centering
		{\includegraphics[width = 0.7\linewidth ]{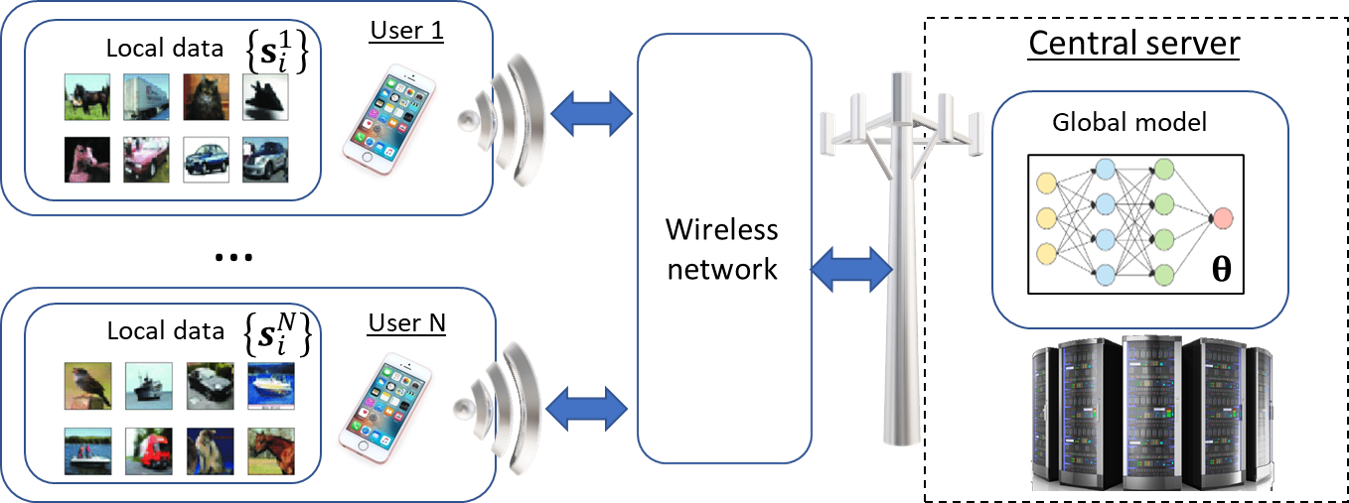}} 
		\caption{ 
		An illustration of the distributed optimization setup. In this example, the data consists of images, where those of user $1$ are biased towards car images, while those of user $N$ contain a large portion of ship images, resulting in a heterogeneous setup. } 
		\vspace{-0.4cm}
		\label{fig:Setup1}	 
	\end{figure*}
		
	When the data is homogeneous, i.e., the local distributions $\{\mathcal{X}_n\}$ are identical, the local loss functions converge to the same expected loss measure on the horizon of a large number of samples $D_n \rightarrow \infty$. However, the statistical heterogeneity of \ac{fl}, i.e., the fact that each user observes data from a different distribution, implies that the parameter vectors which minimize the local loss vary between different users. This property generally affects the behavior of the learning method used in \ac{fl}, such as the common local \ac{sgd} algorithm, detailed in Subsection \ref{subsec:ModelFL}.

	\vspace{-0.2cm}
	\subsection{Communication Channel Model}
	\label{subsec:ModelComm}
	\vspace{-0.1cm} 
    \ac{fl} is often carried out over wireless channels. We consider \ac{fl} setups in which the $N$ users communicate with the server using the same wireless network, either directly or via some wireless access point. As uplink communications, i.e., from the users to the server, is typically notably more constrained as compared to its downlink counterpart in terms of throughput \cite{speedtest2019}, we focus on uplink transmissions over MAC. The downlink channel is modeled as supporting reliable communications at arbitrary rates, as commonly assumed in \ac{fl} studies \cite{alistarh2017qsgd,alistarh2018convergence,amiri2019federated,amiri2020machine,shlezinger2020uveqfed,aji2017sparse,sery2019analog, chang2020communication}.

    We next formulate the uplink channel model. 
    Wireless channels are inherently a shared and noisy media, hence the channel output received by the server at time instance $t$ when each user transmits a $d \times 1$ vector $\myVec{x}_t^n$ is given by 
    \eqspace
    \begin{equation}
        \label{eqn:MAC1}
        \myVec{y}_t = \sum_{n=1}^N 
        \myVec{x}_t^n + \tilde{\myVec{w}}_t,
        \eqspace
    \end{equation}
    where $\tilde{\myVec{w}}_t \sim \mathcal{N}(0, \sigma_w^2 \myMat{I}_d)$ is $d\times 1$ vector of additive noise.
     The channel input is subject to an average power constraint 
     \eqspace
    \begin{equation}
        \label{eqn:power}
    \E\big[ \|\myVec{x}_t^n \|^2\big] \leq P, 
    \eqspace
    \end{equation}
    where $P>0$ represents the available transmission power. The channel in \eqref{eqn:MAC1} represents an additive noise \ac{mac}, whose main resources are its spectral band, denoted $B$, and its temporal blocklength $\tau$, namely, $\myVec{y}_t$ is obtained by observing the channel output over the bandwidth $B$ for a duration of $\tau$ time instances.

    The common approach in wireless communication protocols and in \ac{fl} research is to overcome the mutual interference induced in the shared wireless channels by dividing the bandwidth into multiple orthogonal channels. This can be achieved by, e.g., \ac{fdm}, where the bandwidth is divided into $N$ distinct bands, or via \ac{tdm}, in which the temporal block is divided into $N$ slots which are allocated among the users. In such cases, the server has access to a distinct channel output for each user, of the form
    \eqspace
    \begin{equation}
        \label{eqn:MAC2}
        \myVec{y}_t^n =  \myVec{x}_t^n + \tilde{\myVec{w}}_t^n, \quad n \in \mySet{N}.
        \eqspace
    \end{equation}
    The orthogonalization of the channels in \eqref{eqn:MAC2} facilitates recovery of each $\myVec{x}_t^n $ individually. However, the fact that each user has access only to $1/N$ of the channel resources implies that its throughput, i.e., the volume of data that can be conveyed reliably, is reduced accordingly  \cite[Ch. 4]{goldsmith2005wireless}. In order to facilitate high throughput \ac{fl}, we do not restrict the users to orthogonal communication channels, and thus the server has access to the shared channel output \eqref{eqn:MAC1} rather than the set of individual channel outputs in \eqref{eqn:MAC2}.
     
    We derive our \ac{ota} \ac{fl} scheme and analyze its performance assuming that the users communicate with the server of the noisy \ac{mac} \eqref{eqn:MAC1}. However, in practice wireless channels often induce fading in addition to noise. Each user of index $n$ experiences at time $t$ a block fading channel $\tilde{h}_t^n=h_t^n e^{j\phi_t^n}$, where $h_t^n >0$ and $\phi_t^n\in [-\pi,\pi]$ are its magnitude and phase, respectively.  In such cases, the channel input-output relationship is given by 
    \eqspace
    \begin{equation}
        \label{eqn:MAC1Fading}
        \myVec{y}_t = \sum_{n=1}^N 
        \tilde{h}_t^n 
        \myVec{x}_t^n + \tilde{\myVec{w}}_t.
        \eqspace
    \end{equation}
    Therefore, while our derivation and analysis focuses on additive noise \acp{mac} as in \eqref{eqn:MAC1}, we also show how the proposed \ac{cotaf} algorithm can be extended to fading \acp{mac} of the form \eqref{eqn:MAC1Fading}. In our extension, we assume that the participating entities have \ac{csi}, i.e., knowledge of the fading coefficients. Such knowledge can be obtained by letting the users sense their channels, or alternatively by having the access point/server periodically estimate these coefficients and convey them to the users.
	
	\vspace{-0.2cm}
	\subsection{Local SGD}
	\label{subsec:ModelFL}
	\vspace{-0.1cm} 
    Local \ac{sgd}, also referred to as {\em federated averaging} \cite{mcmahan2016communication}, is a distributed learning algorithm aimed at recovering \eqref{eqn:Global cost}, without having the users share their local data. This is achieved by carrying out multiple training rounds, each consisting of the following three phases:   
    \begin{enumerate}
        \item The server shares its current model at time instance $t$, denoted by $\thetabf_t$, with the users.
        \item Each user sets its local model $\boldsymbol{\theta}_t^n$ to  $\thetabf_t$, and trains it using its local data set over $H$  \ac{sgd} steps, namely, 
    \begin{equation} \label{eq: SGD}
        \boldsymbol{\theta}^n_{t+1} = \boldsymbol{\theta}^n_t - \eta_t \nabla f_{i^n_t}(\boldsymbol{\theta}^n_t),
    \end{equation}
    where $f_{i^n_t}(\thetabf) \triangleq l(s_{i^n_t}^n;\thetabf)$ is the loss evaluated at a single data sample, drawn uniformly from $\{\boldsymbol{s}_i^n\}_{i=1}^{D_n}$, and $\eta_t$ is the \ac{sgd} step size. The update rule \eqref{eq: SGD} is repeated $H$ steps to yield $\boldsymbol{\theta}_{t+H}^n$.
    \item  Each user conveys its trained local model $\boldsymbol{\theta}_{t+H}^n$ (or alternatively, the updates in its trained model $\boldsymbol{\theta}_{t+H}^n - \boldsymbol{\theta}_t^n$) to the central server, which averages them into a global model via\footnote{While we focus here on conventional averaging of the local models, our framework can be naturally extended to weighted averages.} $\thetabf_{t+H} = \frac{1}{N}\sum_{n=1}^N \boldsymbol{\theta}_{t+H}^n$, and sends the new model to the users for another round.
    \end{enumerate}
    
    The uplink transmission in this algorithm is typically executed over an error-free channel with limited throughput, where channel noise and fading are assumed to be eliminated \cite{ li2019convergence,stich2018local, Stich2018sparsified}.
    The local \ac{sgd} algorithm is known to result in a model $\thetabf_t$ whose objective function $F(\thetabf_t)$ converges to $F^\star \triangleq F(\boldsymbol{\theta}^\star)$ as the number of rounds grows for various families of loss measures under homogeneous data \cite{stich2018local}. When the data is heterogeneous, convergence is  affected by an additional term  encapsulating the {\em degree of heterogeneity}, defined as $\Gamma \triangleq F^\star - \frac{1}{N}\sum_{n=1}^N f^\star_n$, where $f_n^\star \triangleq \min_{\thetabf}f_n(\thetabf)$  \cite{li2019convergence}. In particular, for convex objectives, convergence of the global model to \eqref{eqn:Global cost} can be still guaranteed, though at slower rates compared to homogeneous setups  \cite{li2019convergence}. To the best of our knowledge, the convergence of local \ac{sgd} with heterogeneous data carried out over noisy fading wireless channels\YonRep{{\em do coincide with known results when the noise tends to zero?} the answer is yes, we added a discussion about this after the statement of the theorems. We do not include it here as at this stage we have not even formulated the problem so it may be too early to already state that our results generalize previous ones, and we say so after the statement of the theorems.} has not been studied to date. 
    
	\vspace{-0.2cm}
	\subsection{Problem Formulation}
	\label{subsec:ModelProblem}
	\vspace{-0.1cm}    
	Local \ac{sgd}, as detailed in the previous subsection, is the leading learning algorithm in \ac{fl}.
	    Each round of local \ac{sgd} consists of two communication phases: downlink transmission of the global model $\thetabf_t$ from the server to the users, and uplink transmissions of the updated local models $\{\boldsymbol{\theta}_{t+H}^n\}$ from each user to the server. 
         An illustration of a single round of local \ac{sgd} carried out over a wireless \ac{mac} of the form \eqref{eqn:MAC1} is depicted in Fig. \ref{fig:FLOverMAC}.
	\begin{figure*}
		\centering
		{\includegraphics[width = 0.78\linewidth]{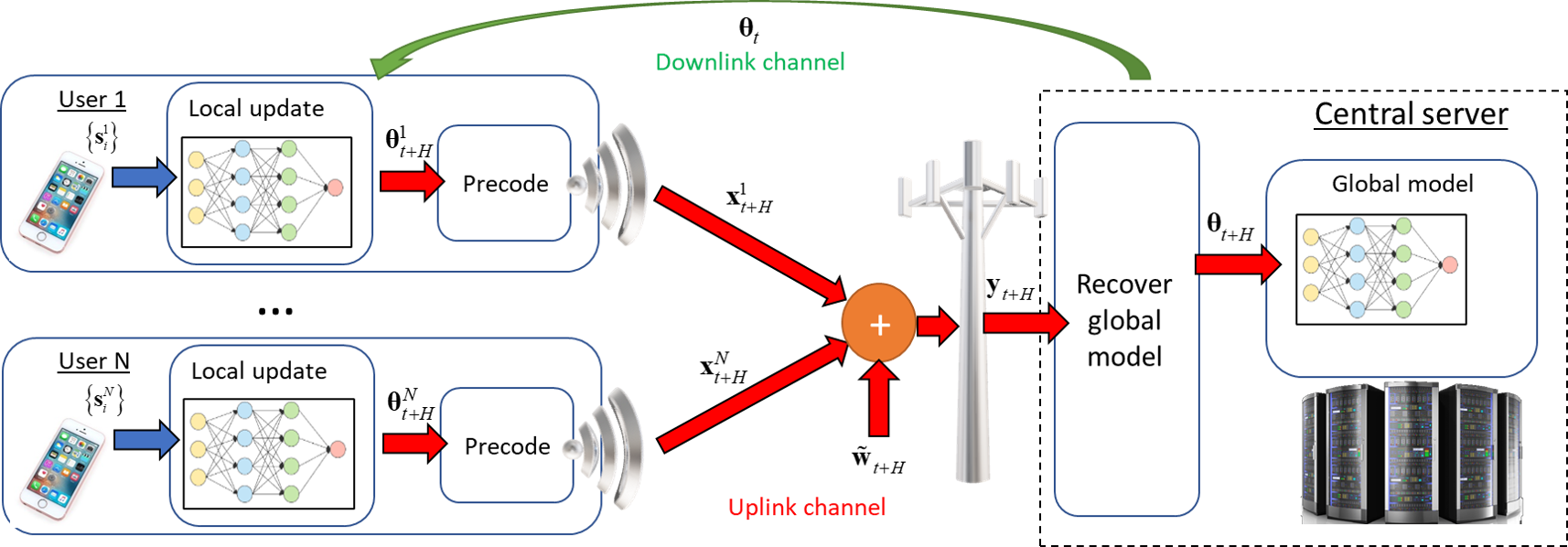}} 
		\caption{An illustration of \ac{fl} over wireless \ac{mac}.} 
		\vspace{-0.4cm}
		\label{fig:FLOverMAC}	 
	\end{figure*}
    This involves the repetitive communication of a large amount of parameters over wireless channels.
    This increased communication overhead is considered one of the main challenges of \ac{fl} \cite{li2019federated,kairouz2019advances}. The conventional strategy in the \ac{fl} literature is to treat the uplink channel as an error-free bit-constrained pipeline,  and thus the focus is on deriving methods for compressing and sparsifying the conveyed model updates, such that convergence of  $\thetabf_t$ to $\boldsymbol{\theta}^\star$ is preserved  \cite{alistarh2017qsgd,shlezinger2020uveqfed,alistarh2018convergence}.  
    However, the model of error-free channels, which are only constrained in terms of throughput, requires the bandwidth of the wireless channel to be divided between the users and have each user utilize coding schemes with a rate small enough to guarantee accurate recovery. This severely limits the volume of data which can be conveyed as compared to utilizing the full bandwidth.
   
   The task of the server on every communication round in \ac{fl} is not to recover each model update individually, but to aggregate them into a global model $\thetabf_t$. This motivates having each of the users exploit the complete spectral and temporal resources by avoiding conventional orthogonality-based strategies and utilizing the wireless \ac{mac} \eqref{eqn:MAC1} on uplink transmissions.  The inherent aggregation carried out by the \ac{mac} can in fact facilitate \ac{fl} at high communication rate via \ac{ota} computations \cite{abari2016over}, as was also proposed in the context of distributed learning in \cite{amiri2019federated,sery2019analog,amiri2020machine}. However, the fact that the channel outputs are corrupted by additive noise is known to degrade the ability of \ac{sgd}-based  algorithms to converge to the desired $\tstar$ for convex objectives \cite{cesa2011online}, adding to the inherent degradation due to statistical heterogeneity. For non-convex objectives, noise can contribute to the overall convergence as it reduces the probability of getting trapped in local minima \cite{Guozhong1995NoiseBackprop,neelakantan2015adding}.  {However, for the learning algorithm to benefit from such additive noise, the level of noise should be limited. It is preferable to have a gradual decay of the noise over time to allow convergence when in the proximity of the desired optimum point, which is not the case when communicating   over noisy \acp{mac}.}
   
    Our goal is to design a communication strategy for \ac{fl} over wireless channels of the form \eqref{eqn:MAC1}. This involves determining a mapping, referred to as precoding, from $\thetabf_t^n$ into $\myVec{x}_t^n$ at each user, as well as a transformation of $\myVec{y}_t$ into $\thetabf_t$ on the server side. 
    The protocol is required to: $1)$ Mitigate the limited convergence of noisy \ac{sgd} for convex objectives by properly precoding the model updates into the channel inputs $\{ \myVec{x}_t^n \}$; $2)$ benefit from the presence of noise when trained using non-convex objectives; and $3)$ allow achieving \ac{fl} performance which approaches that of \ac{fl} over noise-free orthogonal channels for convex objectives, while utilizing the complete spectral and temporal resources of the wireless channel.
    This is achieved by introducing time-varying precoding mapping $\thetabf_t^n \mapsto \myVec{x}_t^n$ at the users' side which accounts for the fact that the parameters are conveyed in order to be aggregated into a global model. Scaling laws are introduced at the server side for accurate transformation of the received signal to a global model. These rules gradually mitigate the effect of the noise on the resulting global model, as detailed in the following section.
    
	\vspace{-0.2cm}
	\section{The Convergent Over-the-Air Federated Learning (COTAF) algorithm}
	\label{sec:Protocol}
	\vspace{-0.1cm} 
	In this section, we propose the COTAF algorithm. 
	We first describe the \ac{cotaf} transmission and aggregation protocol in Subsection~\ref{subsec:ProtocolDescription}. Then, we analyze its convergence in Subsection~\ref{subsec:ProtocolPerformance},  proving its ability to converge to the loss-minimizing network weights under strongly convex objectives. In Subsection~\ref{subsec:ExtensionFadingChannels} we extend \ac{cotaf}  to  fading channels, and discuss its pros and cons   in Subsection~\ref{subsec:ProtocolDiscussion}.

	\vspace{-0.2cm}
	\subsection{Precoding and Reception Algorithm}
	\label{subsec:ProtocolDescription}
	\vspace{-0.1cm} 
	
In \ac{cotaf}, all users transmit their corresponding signals $\{\myVec{x}_t^n\}$ over a shared channel to the server, thus the transmitted signals are aggregated over the wireless \ac{mac} and are received as a sum, together with additive noise, at the server. As in \cite{Liu_Type_2007,  sery2019analog, amiri2020machine}, we utilize analog signalling, namely, each vector $\myVec{x}_t^n$ consists of continuous-amplitude quantities, rather than a set of discrete symbols or bits, as common in digital communications. On each communication round, the server recovers the global model directly from the channel output $\myVec{y}_t$ as detailed in the sequel, and feedbacks the updated model to the users as in conventional local \ac{sgd}.

 \ac{cotaf} implements the local \ac{sgd} algorithm  while communicating over an uplink wireless \ac{mac} as illustrated in Fig. \ref{fig:FLOverMAC}. Let $\mySet{H}$ be the set of time instances in which transmissions occur, i.e., the integer multiples of $H$. In order to convey the local trained model after $H$ local \ac{sgd} steps, i.e., at time instance $t\in \mySet{H}$, the $n$th user precodes its model update $\boldsymbol{\theta}_t^n - \boldsymbol{\theta}_{t-H}^n$ into the \ac{mac} channel input $\myVec{x}_t^n$ via
 \eqspace
\begin{equation}
    \label{eq: transmission form}
    \myVec{x}_t^n = \sqrt{\alpha_t} \left(\boldsymbol{\theta}_t^n - \boldsymbol{\theta}_{t-H}^n\right),
 \eqspace
\end{equation}
where $\alpha_t$ is a precoding factor set to gradually amplify the model updates as   $t$ progresses, while satisfying the power constraint \eqref{eqn:power}. The precoder $\alpha_t$ is given by
 \eqspace
\begin{equation}\label{eq: alpha definition}
        \alpha_t \triangleq \frac{P }{\max_n \E\left[||\thetabf_t^n - \thetabf_{t-H}^n||^2\right]} .
 \eqspace
\end{equation}
The precoding parameter $\alpha_t$ depends on the distribution of the updated model, which depends on the distribution of the data. It can thus be computed by performing offline simulations with smaller data sets and distributing the numerically computed coefficients among the users, as we do in our numerical study in Section \ref{sec:sims}. Alternatively, when the loss function has bounded gradients, this term can be replaced with a coefficient that is determined by the bound on the norm of the gradients\YonRep{{\em Computation of alpha: it is still not clear to me what you do in practice in the simulations and how it is computed. You also say that the analysis holds even if we use the upper bound – did you check that? Finally, you say it does not depend on the data statistics, that is not obvious to me. Can you explain why the expectation does not depend on the data statistics?} - as we say here, in our numerical study we evaluate the coefficient using a smaller Monte Carlo test with the same data. We also numerically checked with the bound though it was also computed using Monte Carlo tests (since we do not have an analytically tractable bound on the gradients), so in Section \ref{sec:sims} we only compute $\alpha_t$ as we state above. As this expression is definitely dependent on the statistics of the data, we no longer claim that our computation is invariant of these statistics.}, as we discuss in Subsection~\ref{subsec:ProtocolDiscussion}.

The channel output \eqref{eqn:MAC1} is thus given by
 \eqspace
\begin{equation} 
    \label{eq: received signal raw} 
    \myVec{y}_t = \sum_{n=1}^{N}\sqrt{\alpha_t}  \left(\boldsymbol{\theta}_t^n - \boldsymbol{\theta}_{t-H}^n\right) + \boldsymbol{\tilde{w}}_t. 
 \eqspace
\end{equation}
In order to recover the aggregated global model $\thetabf_t$ from $\myVec{y}_t $, the server sets
 \eqspace
\begin{equation}
    \label{eq: received signal normalized}
    \thetabf_t = 
     \frac{\myVec{y}_t}{N \sqrt{\alpha_t}} + \thetabf_{t-H},
 \eqspace
\end{equation}
for $t\in \mySet{H}$, where $\boldsymbol{\theta}_0$ is the initial parameter estimate.
The global update rule \eqref{eq: received signal normalized} can be equivalently written as
 \eqspace
\begin{equation}
    \label{eq: global update rule}
    \boldsymbol{\thetabf}_{t} =  \frac{1}{N}\sum_{n=1}^N\boldsymbol{\thetabf}_{t}^n +  \boldsymbol{w}_t,
 \eqspace
\end{equation}
where $\boldsymbol{w}_t \triangleq\frac{\tilde{\boldsymbol{w}}_t}{N\sqrt{\alpha_t}}$ is the equivalent additive noise term distributed via $\myVec{w}_t\sim \mathcal{N}(0,\frac{\sigma_w^2}{N^2 \alpha_t} \boldsymbol{I}_d)$. 
The resulting \ac{ota} \ac{fl} algorithm with $R$ communication rounds is summarized below in Algorithm \ref{alg:Algo1}.  
Here, the local model available at the $n$th user at time $t$ can be written as:
 \eqspace
\begin{equation} \label{eq: theta^n_t update}
    \!\!\thetabf^n_{t\!+\!1}\!=\! 
    \begin{cases}
         \thetabf^n_{t} - \eta_t \nabla f_{i_t^n} (\thetabf^n_t), &  t\!+\!1 \notin \mySet{H}, \\
         \frac{1}{N}\sum\limits_{n=1}^N \left(\thetabf^n_{t} \!-\! \eta_t \nabla f_{i_t^n} (\thetabf^n_t) \right) \!+\! \myVec{w}_t, & t\!+\!1 \in \mySet{H}.\\
    \end{cases}
\end{equation}

 	\begin{algorithm}  
		\caption{ \ac{cotaf} algorithm}
		\label{alg:Algo1}
		\KwData{Fix an initial  $\thetabf_0^n = \thetabf_0$ for each user $n\in\mySet{N}$. }
		\For{$t=1,2,\ldots, RH$}{
			Each user $n\in\mySet{N}$ locally trains $\boldsymbol{\theta}^n_{t}$ via \eqref{eq: SGD}\;
			\If{$t\in\mySet{H}$}{
			    Each user $n\in\mySet{N}$ transmits $\myVec{x}_t^n$ precoded via \eqref{eq: transmission form} over the \ac{mac} \eqref{eqn:MAC1}\;
			    The server recovers $\thetabf_t$ from $\myVec{y}_t$ via \eqref{eq: received signal normalized}\;
			    The server broadcasts $\thetabf_t$ to the users\;
			    Each user $n\in\mySet{N}$ sets $\boldsymbol{\thetabf}^n_{t} = \thetabf_t$\;
			}
		}
		\KwOut{Global model $\thetabf_{RH}$}
	\end{algorithm}

	\vspace{-0.2cm}
	\subsection{Performance Analysis}
	\label{subsec:ProtocolPerformance}
	\vspace{-0.1cm} 
	In this section, we theoretically characterize the convergence of \ac{cotaf} to the optimal model parameters $\tstar$, i.e., the vector $\thetabf$ which minimizes the global loss function. 
	Our analysis is carried out under the following assumptions:
	\begin{enumerate}[label={\em AS\arabic*}] 
		\item \label{itm:As1} The  objective function $F(\cdot)$ is $L$-smooth, namely, for all $\myVec{v}_1, \myVec{v}_2 $ it holds that $F(\myVec{v}_1) -  F(\myVec{v}_2) \leq (\myVec{v}_1-\myVec{v}_2)^T \nabla F(\myVec{v}_2) + \frac{1}{2}L\| \myVec{v}_1-\myVec{v}_2\|^2$.
		\item \label{itm:As2} The objective function $F(\cdot)$ is $\mu$-strongly convex, namely, for all $\myVec{v}_1, \myVec{v}_2$ it holds that $F(\myVec{v}_1) -  F(\myVec{v}_2) \geq (\myVec{v}_1-\myVec{v}_2)^T \nabla F(\myVec{v}_2) + \frac{1}{2}\mu\| \myVec{v}_1-\myVec{v}_2\|^2$. 
		\item \label{itm:As3} The stochastic gradients $\nabla f_{i_t^n}(\boldsymbol{\theta})$ satisfy $\E[ \| \nabla f_{i_t^n}(\boldsymbol{\theta}) \|^2] \leq G^2$ and  $\E[ \| \nabla f_{i_t^n}(\boldsymbol{\theta}) - \nabla f_n(\boldsymbol{\theta}) \|^2] \leq M_n^2$ for some fixed $G^2 > 0$ and $M_n^2 > 0$, for each $\boldsymbol{\theta} \in \Theta$ and $n \in \mathcal{N}$.
	\end{enumerate}
	Assumptions \ref{itm:As1}-\ref{itm:As3} are commonly used when studying the convergence of \ac{fl} schemes, see, e.g., \cite{stich2018local,li2019convergence}. In particular, \ref{itm:As1}-\ref{itm:As2} hold for a broad range of objective functions used in \ac{fl} systems, including  $\ell_2$-norm regularized linear regression and  logistic regression \cite{li2019convergence}, while \ref{itm:As3} represents having bounded second-order statistical moments of the stochastic gradients \cite{stich2018local}. Convergence proofs for such scenarios are of particular interest in the presence of noise, such as that introduced by the wireless channel in \ac{ota} \ac{fl}, as noise is known to degrade the ability of local \ac{sgd} to converge to  $\tstar$. It is also emphasized that these assumptions are required to maintain an analytically tractable convergence analysis, and that  \ac{cotaf}  can be applied for arbitrary learning tasks for which \ref{itm:As1}-\ref{itm:As3} do not necessarily hold, as numerically demonstrated in Section~\ref{sec:sims}.

	After $T=RH$ iterations of updating the global model via \ac{cotaf}, the server can utilize its learned global model for inference. This can be achieved by setting the global model weights according to the instantaneous parameters vector available at this time instance, i.e., $\thetabf_T$. An alternative approach is to utilize the fact that the server also has access to previous aggregated models, i.e., $\{\thetabf_r\}$ for each $r \in \mySet{H}$ such that $r \leq T$. In this case, the server can infer using a model whose parameters are obtained as a weighted average of its previous learned model parameters, denoted by $\hat{\thetabf}_T$, which can be optimized to reduce the model variance \cite{Chen2017checkpoint} and thus improve the convergence rate. 

We next establish a finite-sample bound on the error, given by the expected loss in the objective value at iteration $T$ with respect to  $F^\star$, for both the weighted average model $\hat{\thetabf}_T$ and instantaneous weights $\thetabf_T$. 
We begin with the bound relevant for the average model, stated in the following theorem:
	\begin{theorem}\label{th: theorem 1} 
	Let  $\{\boldsymbol{\theta}_t^n\}_{n=1}^N$ be the model parameters generated by \ac{cotaf} according to \eqref{eq: SGD}, and \eqref{eq: global update rule} over $R$ rounds, i.e., $t \in \{0,1,\ldots T-1\}$ with $T= RH$.
	Then, when \ref{itm:As1}-\ref{itm:As3} hold and the step sizes are set to $\eta_t = \frac{4}{\mu(a+t)}$ with shift parameter  $a>\max\{16\frac{L}{\mu}, H\} $, and the precoder is set as in \eqref{eq: alpha definition}, it holds that
	\begin{align} 
    &\E[F(\hat{\thetabf}_T)] - F^\star \leq \frac{4(T+R)}{3\mu S_R}(2a+H+R-1)B  \notag \\ 
    & \quad+ \frac{16d TH G^2\sigma_w^2}{3\mu P N^2 S_R} (2a\!+\! T\! +\! H)\! +\!  \frac{\mu a^3}{6S_R}||\thetabf_0\! -\!\tstar||^2, 
    \label{eq: theorem 1}
    \end{align}
    where $\hat{\thetabf}_T = \frac{1}{S_R} \sum_{r=1}^R \beta_r\thetabf_{r\SGDIter}$, for $\beta_t = (a+t)^2$, $S_R =\sum_{r=1}^R \beta_{r\SGDIter} \!\geq \!\frac{1}{3H}T^3 $, and
\ifFullVersion    
    \begin{equation}
    \label{eqn:Bdef}
        B  = 8H^2G^2 +\frac{1}{N^2}\sum_{n=1}^N M_n^2 + 6L\Gamma.
    \end{equation}
\else
    $B\! = \!8H^2G^2 \!+ \!\frac{1}{N^2}\sum_{n=1}^N M_n^2\! +\! 6L\Gamma$.
\fi
	\end{theorem} 

	\begin{IEEEproof}
		The proof is given in Appendix \ref{app:Proof1}. 
	\end{IEEEproof}
	
	\smallskip
    The weighted average in $\hat{\thetabf}_T$  is taken over the models known to the server, i.e., $\{\thetabf_r\}$ with $r \in \mySet{H}$. For comparison, in previous convergence studies of local \ac{sgd} and its variants \cite{stich2018local,Stich2018sparsified}, the weighted average is computed over every past model, including those available only to users and not to the server. In such cases, the resulting bound does not necessarily correspond to an actual model used for inference, since the weighted average is not attainable. Comparing Theorem \ref{th: theorem 1} to the corresponding result in \cite{stich2018local}, which considered i.i.d data and noise-free channels, we observe that \ac{cotaf} achieves the same convergence rate, with an additional  term which depends on the noise-to-signal ratio $\sigma_w^2/P$, and decays as $1/T$ as discussed in the sequel. When $\sigma_w^2/P = 0$, Theorem \ref{th: theorem 1} specializes into \cite[Thm 2.2]{stich2018local}.

	In the next theorem, we establish a finite sample bound on the error for the instantaneous weights $\thetabf_T$ rather than the weighted average $ \hat{\thetabf}_T$:
	\begin{theorem}
	\label{th: theorem 2} 
    Let  $\{\boldsymbol{\theta}_t^n\}_{n=1}^N$ be the model parameters generated by \ac{cotaf} according to \eqref{eq: SGD} and \eqref{eq: global update rule} over $R$ rounds, i.e., $t \in \{0,1,\ldots T-1\}$ with $T= RH$.
	Then, when \ref{itm:As1}-\ref{itm:As3} hold and the step sizes are set to $\eta_t = \frac{2}{\ConvParam(\gamma+t)}$, for $\gamma \geq \max(\frac{8L\StepSizeNume}{\ConvParam} ,H)$, it holds that: 
	\begin{equation}\label{eq: theorem 2}
	    \E[\Objective(\myWeights_{T})  ] -  \Objective(\myWeights\Opt)\leq \frac{2\SmoothParam \max \big(4 C,\ConvParam^2 \gamma \delta_{0} \big)  }{\ConvParam^2(T+ \gamma)}.
	\end{equation} 
	where $C = B + \frac{4dH^2 G^2 \sigma_w^2}{P N^2}$.   
	\end{theorem}
	
		\begin{IEEEproof}
		The proof is given in Appendix \ref{app:Proof2}. 
	\end{IEEEproof}
	
		\smallskip
	The proofs for both Theorems \ref{th: theorem 1}-\ref{th: theorem 2} follow the same first steps. Yet in the derivation of Theorem \ref{th: theorem 2} an additional relaxation was applied, implying that the bound in \eqref{eq: theorem 2} is less tight than \eqref{eq: theorem 1}. For the noise-free case, i.e.,  $\sigma_w^2/P = 0$, Theorem \ref{th: theorem 2} coincides with \cite[Thm. 1]{li2019convergence}.

	    	Theorems \ref{th: theorem 1} and \ref{th: theorem 2} characterize of the effect of three sources of error on the rate of convergence: The accuracy of the initial guess initial distance $\thetabf_0$; the effect of statistical heterogeneity encapsulated in $\Gamma$, which is linear in $B$ and $C$;  and the noise-to-signal ratio $\sigma_w^2/P$  induced by the wireless channel. In particular, in \eqref{eq: theorem 2} all of these quantities, which potentially degrade the accuracy of the learned global model, contribute to the error bound in a manner proportional to $1/(T+\gamma)$, i.e., which decays as the number of rounds grows. The same observation also holds for \eqref{eq: theorem 1}, in which the aforementioned terms contribute in a manner that decays at an order proportional to $1/T$. The fact that the error due to the noise, encapsulated in $\sigma_w^2/P$, decays with the number of iterations, indicates the ability of \ac{cotaf} to mitigate the harmful effect of the \ac{mac} noise, as discussed next.  
	    	
 	Comparing \eqref{eq: theorem 2} to the corresponding bound for local \ac{sgd} with heterogeneous data and without communication constraints in \cite[Thm. 1]{li2019convergence}, i.e., over orthogonal channels as in \eqref{eqn:MAC2} without noise, we observe that the bound takes a similar form as that in \cite[Eq. (5)]{li2019convergence}. The main difference is in the additional term that depends on the noise-to-signal ratio $\sigma_w^2/P$ in the constant $C$, which does not appear in the noiseless case in \cite{li2019convergence}. Consequently, the fact that \ac{cotaf} communicates over a noisy channel induces an additional term that can be written as  $\sigma_w^2/P$ times some factor which, as the number of \ac{fl} rounds $R$ grows, is dominated by $\frac{H^2}{N^2(T+\gamma)}$. This implies that the time-varying precoding and aggregation strategy implemented by \ac{cotaf} results in a gradual decay of the noise effect, and allows its contribution to be further mitigated by increasing the number of users $N$. Furthermore, Theorems \ref{th: theorem 1}-\ref{th: theorem 2}  yield the same asymptotic convergence rate to that observed for noiseless local \ac{sgd} in \cite{li2019convergence}, as stated in the following corollary:
	
	\begin{corollary}
	\label{cor:asymptotic}
	\ac{cotaf} achieves an asymptotic convergence rate of $\mySet{O}(\frac{1}{T})$. 
	\end{corollary} 	
	{\em Proof:}
	The corollary follows directly from \eqref{eq: theorem 1} and \eqref{eq: theorem 2} by letting $T$ grow arbitrarily large while keeping the number of \ac{sgd} iterations per round $H$ fixed. 
	\qed

		\smallskip
	Corollary \ref{cor:asymptotic} implies that \ac{cotaf} allows \ac{ota} \ac{fl} to achieve the same asymptotic convergence rate as local \ac{sgd} with a strongly convex objective and without communication constraints \cite{stich2018local,li2019convergence}. This advantage of \ac{cotaf} adds to its ability to exploit the temporal and spectral resources of the wireless channel, allowing communication at higher throughput compared to conventional designs based on orthogonal communications, as discussed in Subsection \ref{subsec:ProtocolDiscussion}.

	\vspace{-0.2cm}
	\subsection{Extension to Fading Channels}
	\label{subsec:ExtensionFadingChannels}
	\vspace{-0.1cm} 
    In the previous subsections we focused on \ac{fl} over wireless channels modeled as noisy \acp{mac} \eqref{eqn:MAC1}. For such channels we derived \ac{cotaf} and characterized its convergence profile. We next show how \ac{cotaf} can be extended to fading \acp{mac} of the form \eqref{eqn:MAC1Fading}, while preserving its proven convergence. As detailed in Subsection \ref{subsec:ModelComm}, we focus on scenarios in which the participating entities have \ac{csi}.

    In fading \acp{mac}, the signal transmitted by each user undergoes a  fading coefficient denoted $h_t^n e^{j\phi_t^n}$  \eqref{eqn:MAC1Fading}. 
    Following the scheme proposed in \cite{amiri2019federated} for conveying sparse model updates, each user can utilize its \ac{csi} to cancel the fading effect by amplifying the signal by its inverse channel coefficient. However, weak channels might cause an arbitrarily high amplification, possibly violating the transmission power constraint \eqref{eqn:power}. Therefore, a threshold $h_{min}$ is set, and users observing fading coefficients of a lesser magnitude than  $h_{min}$ do not transmit in that communication round. As  channels typically attenuate their signals, it holds that $h_{min} < 1$.  Under this extension of \ac{cotaf},  \eqref{eq: transmission form} becomes 
	    \begin{equation}
	        \myVec{x}_t^n = 
	        \begin{cases}
                  \frac{\sqrt{\alpha_t}h_{min}}{h_t^n}e^{-j\phi_t^n} \left(\boldsymbol{\theta}_t^n - \boldsymbol{\theta}_{t-H}^n \right), &  h_t^n > h_{min}, \\
                  0, &  h_t^n \leq h_{min}.
            \end{cases}
            \label{eqn:CensAlg}
	    \end{equation}
	   Here, $e^{-j\phi_t^n} $ is a phase correction term as in \cite{sery2019analog}. Note that the energy constraint \eqref{eqn:power} is preserved as $\E[\|\myVec{x}_t^n\|^2] \leq P$.
       
     To formulate the server aggregation, we let $\mySet{K}_t \subset\mySet{N}$ be the set of user indices whose corresponding channel at time $t$ satisfies $h_t^n > h_{min}$. As the server has \ac{csi}, it knows $\mySet{K}_t$, and can thus recover the aggregated model $\thetabf_t$ in a similar manner as in \eqref{eq: received signal normalized}-\eqref{eq: global update rule} via $ \thetabf_t = \frac{\myVec{y}_t}{|\mySet{K}_t| \sqrt{\alpha_t} h_{min}} + \thetabf_{t-H}$, i.e., 
 \eqspace
         \begin{align}
        \thetabf_t 
                 &
                 = \frac{1}{|\mySet{K}_t|}\sum_{n \in \mySet{K}_t} \thetabf_t^n + \frac{N}{|\mySet{K}_t| h_{min}}\myVec{w}_t.
        \label{eq: received signal normalized fading}
 \eqspace
    \end{align}
    Comparing \eqref{eq: received signal normalized fading} to the corresponding equivalent formulation  in \eqref{eq: global update rule}, we note that the proposed extension of \ac{cotaf} results in two main differences from the fading-free scenario: $1)$ the presence of fading is translated into an increase in the noise power, encapsulated in the constant $\frac{N}{|\mySet{K}_t| h_{min}} > 1$; and $2)$ less models are aggregated in each round as $|\mySet{K}_t| \leq N$. 
    The set of participating users $\mySet{K}_t$ depends on the distribution of the fading coefficients. Thus, in order to analytically characterize how the convergence is affected by fading compared to the scenario analyzed in Subsection \ref{subsec:ProtocolPerformance}, we introduce the following assumption: 
       
    \begin{enumerate}[resume, label={\em AS\arabic*}]
        \item \label{itm:As4} At each communication round, the participating users set $\mySet{K}_t$ contains $K \leq N$ users and is uniformly distributed over all the subsets of $\mySet{N}$ of cardinality $K$.
    \end{enumerate}
    
    Note that Assumption \ref{itm:As4} can be imposed by a simple distributed mechanism using an opportunistic carrier sensing \cite{cohen2010time}. Specifically, each user maps its $h_t^n$ to a backoff time $b_t^n$ based on a predetermined common function $f(h)$, which is a decreasing function with $h$ (truncated at $h_{min}$). Then, each user with $h_t^n\geq h_{min}$ listens to the channel and transmits a low-power beacon when its backoff time expires, which can be sensed by other users. If $K$ transmissions have been identified, the corresponding $K$ users transmit their data signal to the server. Otherwise, the users wait (which occurs with a small probability as $N$ increases, and $h_{min}$ decreases) to the next time step. This mechanism guarantees $|\mySet{K}_t|=K$ at each update. We point out that Assumption \ref{itm:As4} is needed for theoretical analysis only. 
    \ifFullVersion
    In practice, COTAF achieves a similar convergence property when implementing it without this mechanism, i.e.,  when $|\mySet{K}_t|$ is random.
    \fi
    
    Next, we  characterize the convergence of the instantaneous global model, as stated in the following theorem:

    \begin{theorem}\label{th: theorem 3} 
	Let  $\{\boldsymbol{\theta}_t^n\}_{n=1}^N$ be the model parameters generated by the extension of \ac{cotaf} to fading channels over $R$ rounds, i.e., $t \in \{0,1,\ldots T-1\}$ with $T= RH$. 
	Then, when \ref{itm:As1}-\ref{itm:As4} hold and the step sizes are set to $\eta_t = \frac{2}{\ConvParam (\gamma+t)}$, for $\gamma \geq \max(\frac{8L\StepSizeNume}{\ConvParam},H)$, it holds that: 
 \eqspace
	\begin{equation}\label{eq: theorem 3}
	    \E[\Objective(\myWeights_{T})  ] -  \Objective(\myWeights\Opt)\leq \frac{2\SmoothParam \max \big(4 (\tilde{C}+D),\ConvParam^2 \gamma \delta_{0} \big)  }{\ConvParam^2(T+ \gamma)}.
 \eqspace
	\end{equation} 
	where $\tilde{C} = B + \frac{4dH^2 G^2 \sigma_w^2}{P K^2 h_{min}^2}$ and $D= \frac{4(N-K)}{K(N-1)} H^2 G^2$.
	\end{theorem} 
	\begin{IEEEproof}
		The proof is given in Appendix \ref{app:Proof3}. 
	\end{IEEEproof}
	
	\smallskip
	Comparing Theorem \ref{th: theorem 3} to the corresponding convergence bound for fading-free channels in Theorem \ref{th: theorem 3} reveals that the extension of \ac{cotaf} allows the trained model to maintain its asymptotic convergence rate of $O(\frac{1}{T})$ also in the presence of fading channel conditions. However, the aforementioned differences in the equivalent global model due to fading are translated here into additive terms increasing the bound on the distance between the expected instantaneous objective $\E[\Objective(\myWeights_{T})]$ and its desired optimal value. In particular, the fact that not all users participate in each round induces the additional positive term $D$ in \eqref{eq: theorem 3}, which equals zero when $K=N$ and grows as $K$ decreases. Furthermore, the increased equivalent noise results in the additive term $\tilde{C}$ being larger than the corresponding symbol $C$ in \eqref{eq: theorem 2} due to the increased equivalent noise-to-signal ratio which stems from the scaling by $h_{min}$ at the precoder and the corresponding aggregation at the server side. Despite the degradation due to the presence of fading,   \ac{cotaf} is still capable of guaranteeing convergence and approach the performance of fading and noise-free local \ac{sgd} when training in light of a smooth convex objective in a federated manner, as also numerically observed in our simulation study in Section \ref{sec:sims}. 
	     
	\vspace{-0.2cm}
	\subsection{Discussion}
	\label{subsec:ProtocolDiscussion}
	\vspace{-0.1cm} 
		\ac{cotaf} is designed to allow \ac{fl} systems operating over shared wireless channels to exploit the full spectral and temporal resources of the media. This is achieved by accounting for the task of aggregating the local models into a global one as a form of \ac{ota} computation \cite{abari2016over}. Unlike conventional orthogonality-based transmissions, such as \ac{fdm} and \ac{tdm}, in \ac{ota} \ac{fl} the available band and/or transmission time of each user does not decrease with the number of users $N$, allowing the simultaneous participation of a large number of users without limiting the throughput of each user. 
		Compared to previous strategies for \ac{ota} \ac{fl}, \ac{cotaf} allows the implementation of local \ac{sgd}, which is arguably the most widely used \ac{fl}  scheme, over wireless \acp{mac} with proven convergence. This is achieved without having to restrict the model updates to be sparse with an identical sparsity pattern shared among all users \cite{amiri2020machine,amiri2019federated}, or requiring the users to repeatedly compute the gradients over the full data set as in~\cite{sery2019analog}. 
	
		A major challenge in implementing \ac{sgd} as an \ac{ota} computation stems from the presence of the additive channel noise, whose contribution does not decay over time \cite{cesa2011online}. Under strongly convex objectives, noisy distributed learning can be typically shown to asymptotically converge to some distance from the minimal achievable loss, unlike noise-free local \ac{sgd} which is known to converge to desired $F^\star$ at a rate of $\mySet{O}(\frac{1}{T})$ \cite{stich2018local}. \ac{cotaf} involves additional precoding and scaling steps which result in an effective decay of the noise contribution, thus allowing to achieve convergence results similar to noise-free local \ac{sgd} with strongly convex objectives while operating over shared noisy wireless channels. The fact that \ac{cotaf} mitigates the effect of noise in a gradual manner allows benefiting from the advantages of such noise profiles under non-convex objectives, where a controllable noise level was shown to facilitate convergence by reducing the probability of the learning procedure being trapped in local minima \cite{Guozhong1995NoiseBackprop,neelakantan2015adding}. This behavior is numerically demonstrated in Section~\ref{sec:sims}.

	\ac{cotaf} consists of an addition of simple precoding and scaling stages to local \ac{sgd}. This precoding stage is necessary for assuring a steady convergence rate, while keeping power consumption under control.  Implementing the time-varying precoding in \eqref{eq: alpha definition} implies that every user has to know  $\max_n \E\left[||\thetabf_t^n - \thetabf_{t-H}^n||^2\right]$, for each communication round $t \in \mySet{H}$. When operating with a decaying step size, as is commonly required in \ac{fl}, and when \ref{itm:As3} holds, this term is upper bounded by $H^2 \eta_{t-H}^2G^2$ (see Lemma \ref{lem: lemma 2} in Appendix \ref{app:Proof1}), and the upper bound can be used instead in \eqref{eq: alpha definition}, while maintaining the convergence guarantees of Theorems \ref{th: theorem 1}-\ref{th: theorem 2}. Alternatively, since $\alpha_t$  should be proportional to the inverse of the maximal difference of consecutively transmitted models, one can numerically estimate these values by performing offline simulation over a smaller data set. Once these values are numerically computed,  the server can distribute them to the users over the downlink channel. 

  \ac{cotaf} involves analog transmissions over \ac{mac}, {which allows the superposition carried out by the \ac{mac} to aggregate the parameters as required in \ac{fl}}. As a result, \ac{cotaf} is subject to the challenges associated with such signalling, e.g., the need for accurate synchronization among all users. 
    Finally, \ac{ota} \ac{fl} schemes such as \ac{cotaf} require the participating users to share the same wireless channel, i.e., reside in the same geographical area, while \ac{fl} systems can be trained using data aggregated from various locations. We conjecture that \ac{cotaf} can be combined in multi-stage \ac{fl}, such as clustered \ac{fl} \cite{shlezinger2020clustered}. We leave this for future study. 

	\vspace{-0.2cm}
	\section{Numerical Evaluations}
	\label{sec:sims}
	\vspace{-0.1cm} 

In this section, we provide numerical examples to illustrate the performance of \ac{cotaf} in two different settings. We begin with a scenario of learning a linear predictor of the release year of a song from audio features in Subsection \ref{subsec: Millions song sims}. In this setup, the objective is strongly convex, and the model assumptions under which COTAF is analyzed hold. 
In the second setting detailed in Subsection \ref{subsec:cifar}, we consider a more involved setup, in which the loss surface with respect to the learned weights is not convex. Specifically, we train a CNN for classification on the  CIFAR-10 dataset. 

	\vspace{-0.2cm}
\subsection{Linear Predictor Using the Million Song Dataset} \label{subsec: Millions song sims}
	\vspace{-0.1cm}
We start by examining \ac{cotaf} for learning how to predict the release year of a song from audio features in an FL manner. We use the 
\ifFullVersion
dataset available by the UCI Machine Learning Repository \cite{Lichman:2013}, extracted from the Million Song Dataset collaborative project between The Echo Nest and LabROSA \cite{Bertin-Mahieux2011}. The Million Song Dataset 
\else
Million Song Dataset \cite{Bertin-Mahieux2011}, that
\fi 
contains songs which are mostly western, commercial tracks ranging from 1922 to 2011. Each song is associated with a release year and $90$ audio attributes. Consequently, each data sample $\myVec{s}$ takes the form $\myVec{s} = \{\myVec{s}_s, s_y\}$, where $\myVec{s}_s$ is the audio attributes vector and $s_y$ is the year. 
The system task is to train a linear estimator $\thetabf$ with $d=90$ entries in an FL manner using data available at $N$ users, where each user has access to $D_n = 9200$ samples. The predictor is trained using the regularized linear least-squares loss, given by:
\begin{equation}
f(\boldsymbol{\theta},  \{\myVec{s}_s, s_y\}) = \frac{1}{2}(\boldsymbol{s}_s^T\boldsymbol{\theta}-s_y)^2 + \frac{\lambda}{2}||\boldsymbol{\theta}||^2,
\label{eqn:LossLS}
\end{equation}
where we used $\lambda=0.5$.
We note that the loss measure \eqref{eqn:LossLS} is strongly convex and has a Lipschitz gradient, and thus satisfies the conditions of Theorem \ref{th: theorem 1}. In every \ac{fl} round, each user performs $H$ \ac{sgd} steps \eqref{eq: SGD} where the step size is set via Theorem \ref{th: theorem 1}. In particular, the parameters  $L$ and $\mu$ are numerically evaluated before transmitting the model update to the server over the \ac{mac}. The precoding coefficient $\alpha_t$ is computed via \eqref{eq: alpha definition} using numerical averaging, i.e., we carried out an offline simulation of local \ac{sgd} without noise and with $20\%$ of the data samples, and computed the averaged norm of the resulting model updates.  

 \begin{figure}
\begin{center}
    {\includegraphics[height=\figHeight,width=\figWidth ]{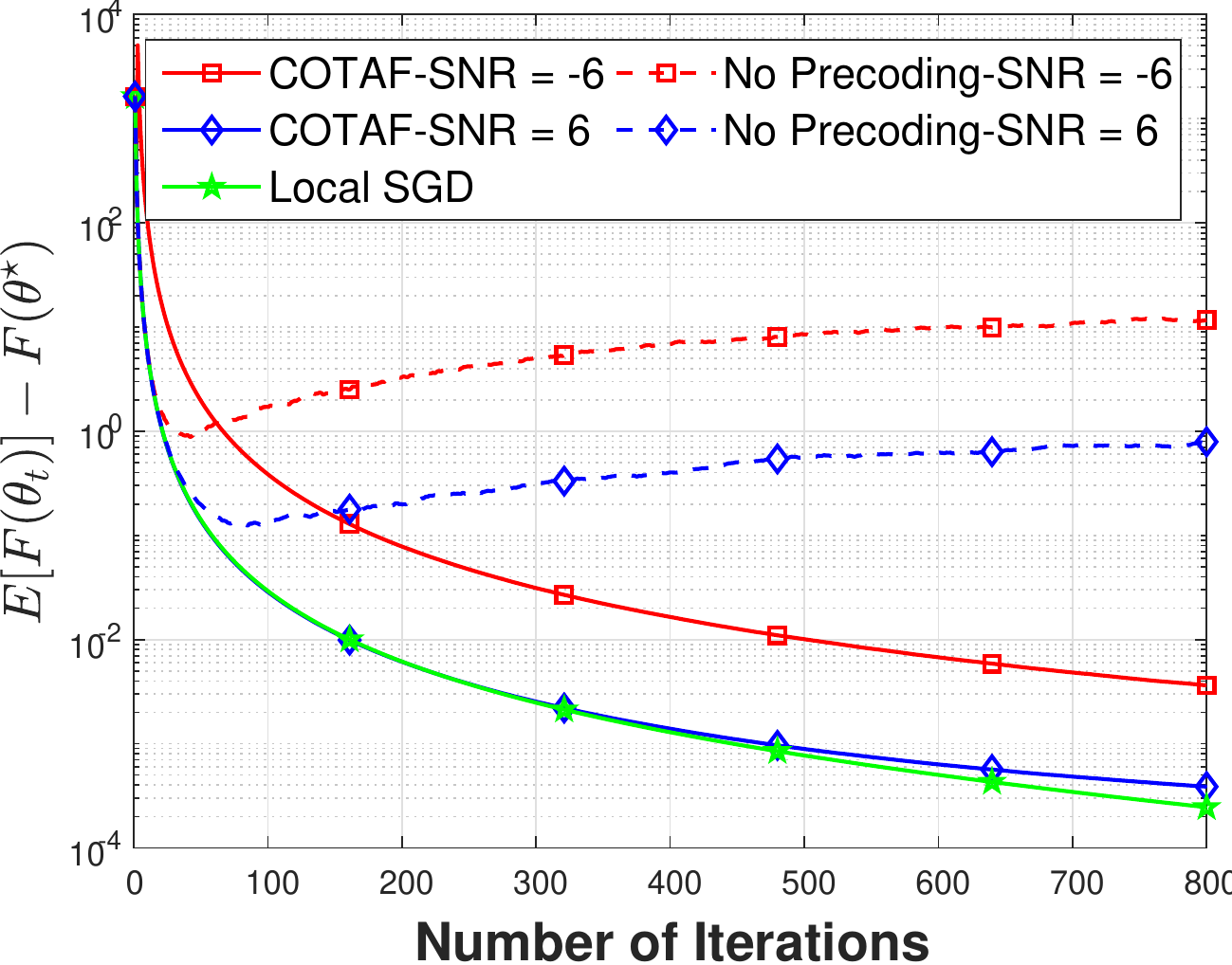}}
  \caption{Linear predictor, Million Song dataset, $H=40$, $N = 50$.
  }
  \label{fig: COTAF_Year_Pred}
\end{center}
  \end{figure} 

 We numerically evaluate the gap from the achieved expected objective and the loss-minimizing one, i.e., $\E\big[F(\thetabf_t)\big] - F^\star$. Using this performance measure, we compare \ac{cotaf}  to the following \ac{fl} methods: (i) Local \ac{sgd}, in which every user conveys its model updates over a noiseless individual channel; (ii) Non-precoded \ac{ota} \ac{fl}, where every user  transmits its model updates over the \ac{mac} without time-varying precoding  \eqref{eq: alpha definition} and with a constant amplification as in \cite{sery2019analog}, i.e., $\myVec{x}_t^n = P (\thetabf_t^n - \thetabf_{t-H}^n)$. The stochastic expectation is evaluated by averaging over $50$ Monte Carlo trials, where in each trial the initial $\thetabf_0$ is randomized from zero-mean Gaussian distribution with covariance $5\myMat{I}_d$. 
 
\begin{figure}
\begin{center}
    {\includegraphics[height=\figHeight,width=\figWidth ]{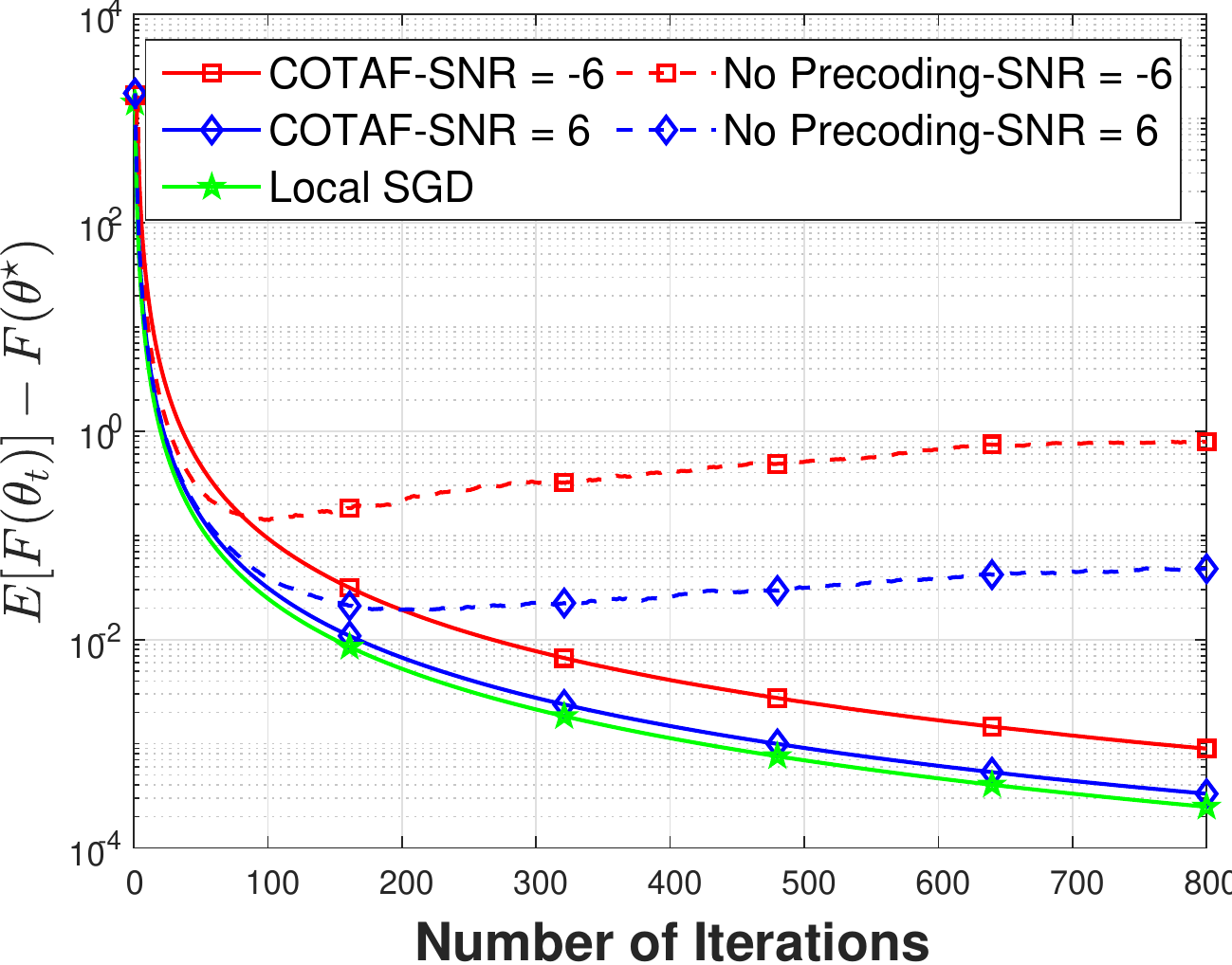}}
  \caption{Linear predictor, Million Song dataset, $H=40$, $N = 200$.
  } 
  \vspace{-0.4cm}
  \label{fig: COTAF_Year_Pred2}
\end{center}
  \end{figure}
 
We simulate \acp{mac} with signal-to-noise ratios (SNRs) of $P/\sigma_w^2 =-6$dB and $P/\sigma_w^2 =6$dB. In Fig. \ref{fig: COTAF_Year_Pred}, we present the performance evaluation when the number of users is set to $N=50$ and the number of \ac{sgd} steps is $H=40$. It can be seen in Fig. \ref{fig: COTAF_Year_Pred} that \ac{cotaf} achieves performance within a minor gap from that of local \ac{sgd} carried out over ideal orthogonal noiseless channels. This improved performance of \ac{cotaf} is achieved without requiring the users to divide the spectral and temporal channel resources among each other, thus to communicate at higher throughput uplink communications as compared to the local SGD.  
This is due to the precoding scheme of \ac{cotaf}, which allows gradually mitigating the effect of channel noise, while \ac{ota} \ac{fl} without such time-varying precoding results in a dominant error floor due to presence of non-vanishing noise.
 
\begin{figure}
\begin{center}
  {\includegraphics[height=\figHeight,width=\figWidth ]{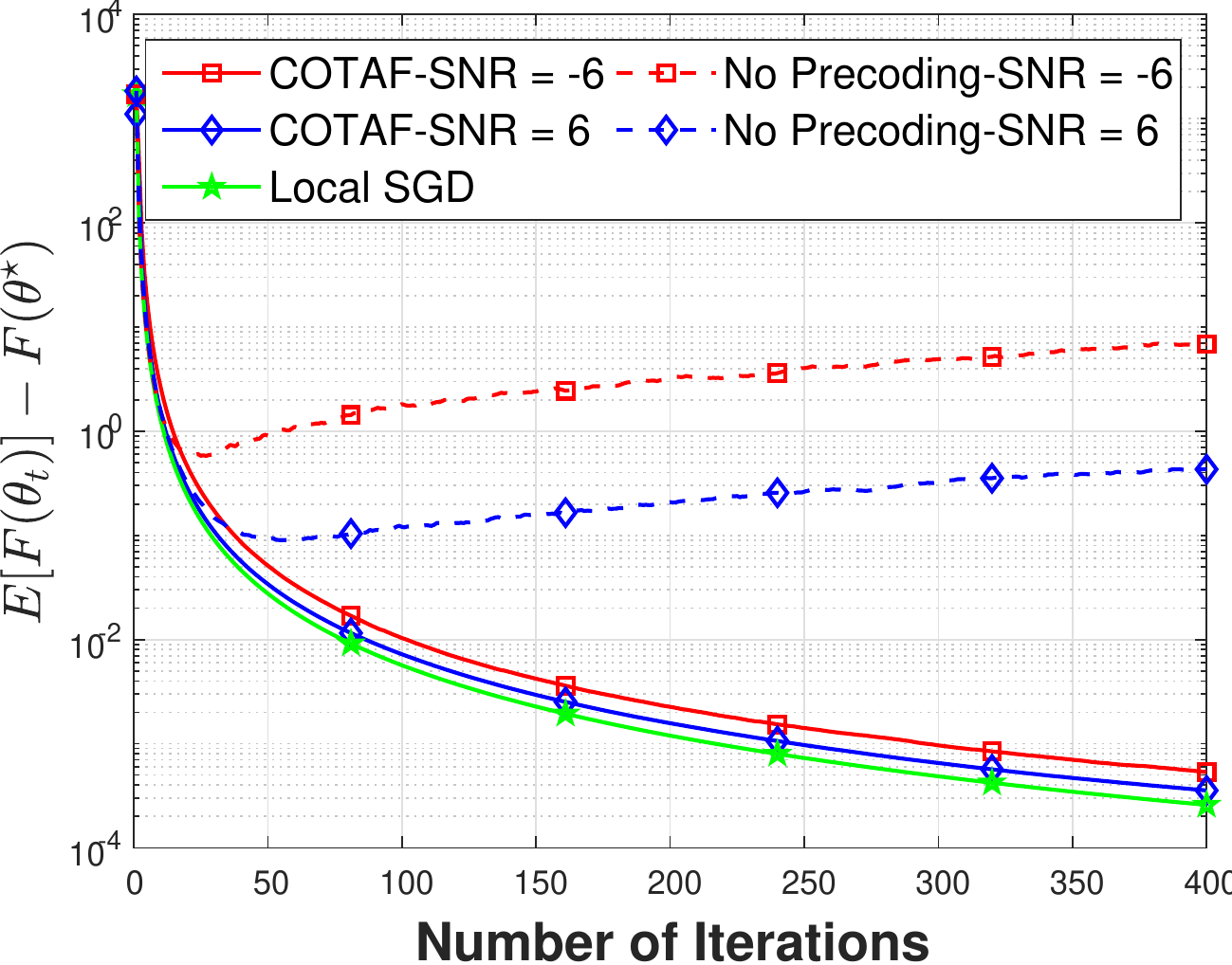}}  
  \caption{Linear predictor, Million Song dataset, $H=80$, $N = 50$.
  }
  \label{fig: COTAF_Year_Pred3}
\end{center}
  \end{figure} 

Next, we repeat the simulation study of Fig. \ref{fig: COTAF_Year_Pred} while increasing the number of users to be $N=200$ in Fig.~\ref{fig: COTAF_Year_Pred2}, and with setting the number of \ac{sgd} steps to $H=80$ in Fig.~\ref{fig: COTAF_Year_Pred3}.   The number of gradient computations $T= RH$ and the overall number of training samples $N D_n$ is kept constant throughout the simulations. 
 Figs. \ref{fig: COTAF_Year_Pred2}-\ref{fig: COTAF_Year_Pred3} thus demonstrate the dependence of \ac{cotaf} performance on two key system parameters: The number of users, $N$, and the number of SGD steps, $H$, between communication rounds.
 
 Observing Fig. \ref{fig: COTAF_Year_Pred2} and comparing it to Fig. \ref{fig: COTAF_Year_Pred}, we note that increasing the number of users improves the performance of both \ac{ota} \ac{fl} schemes, despite the fact that each user holds less training samples. In particular, \ac{cotaf} effectively coincides with the performance of noise-free local \ac{sgd} here, while the non-precoded \ac{ota} \ac{fl} achieves an improved performance as compared to the setting with $N=50$, yet it is still notably outperformed by \ac{cotaf}.  The gain in increasing the number of users follows from the fact that averaging over a larger number of users at the server side mitigates the contribution of the channel noise, as theoretically established for \ac{cotaf} in Subsection \ref{subsec:ProtocolPerformance}. It is emphasized that when using orthogonal transmissions, as implicitly assumed when using conventional local \ac{sgd}, increasing the number of users implies that the channel resources must be shared among more users, hence the throughput of each users decreases. However, in \ac{ota} \ac{fl} the throughput is invariant of the number of users. Comparing Fig.~\ref{fig: COTAF_Year_Pred3} to Fig.~\ref{fig: COTAF_Year_Pred} reveals that increasing the \ac{sgd} steps $H$ can improve the performance of \ac{ota} \ac{fl} as the channel noise is induced less frequently. Nonetheless, the gains here are far less dominant than those achievable by allowing more users to participate in the \ac{fl} procedure, as observed in  Fig. \ref{fig: COTAF_Year_Pred2}. 
The results depicted in Figs. \ref{fig: COTAF_Year_Pred}-\ref{fig: COTAF_Year_Pred3} demonstrate the benefits of \ac{cotaf}, as an \ac{ota} \ac{fl} scheme which accounts for both the convergence properties of local \ac{sgd} as well as the unique characteristics of wireless communication channels. 
 
Next, we simulate the effect of fading channels on \ac{cotaf}. 
In particular, we apply the extension of \ac{cotaf} to fading channels detailed in Subsection \ref{subsec:ExtensionFadingChannels}. 
Here, the \ac{mac} input-output relationship is given by \eqref{eqn:MAC1Fading}, and block fading channel coefficients $\{h_t^n\}$ are sampled from a Rayleigh distribution in an i.i.d. fashion, while the remaining parameters are the same as those used in the scenario simulated in Fig.\ref{fig: COTAF_Year_Pred}. The threshold $h_{min}$ in \eqref{eqn:CensAlg} is set such that on average $40$ out if the $N=50$ users participate in each communication round.
 
For fairness, the same conditions are applied in the no precoding setting, i.e., the users utilize their \ac{csi} to cancel the effect of the channel as in \cite{amiri2019federated}. This comparison allows us to illustrate the significance of the dedicated precoding introduced by \ac{cotaf}.
 
The results, depicted in Fig. \ref{fig: COTAF_Year_Pred_fadingChannels}, demonstrate that \ac{cotaf} maintains its ability to approach the performance of noise-free local \ac{sgd}, observed in Figs. \ref{fig: COTAF_Year_Pred}-\ref{fig: COTAF_Year_Pred3} for additive noise \acp{mac}. This demonstrates the ability the extended \ac{cotaf} detailed in Subsection \ref{subsec:ExtensionFadingChannels} to preserve its improved convergence properties in  fading channels.    
\begin{figure}
    \begin{center}
    {\includegraphics[height=\figHeight,width=\figWidth ]{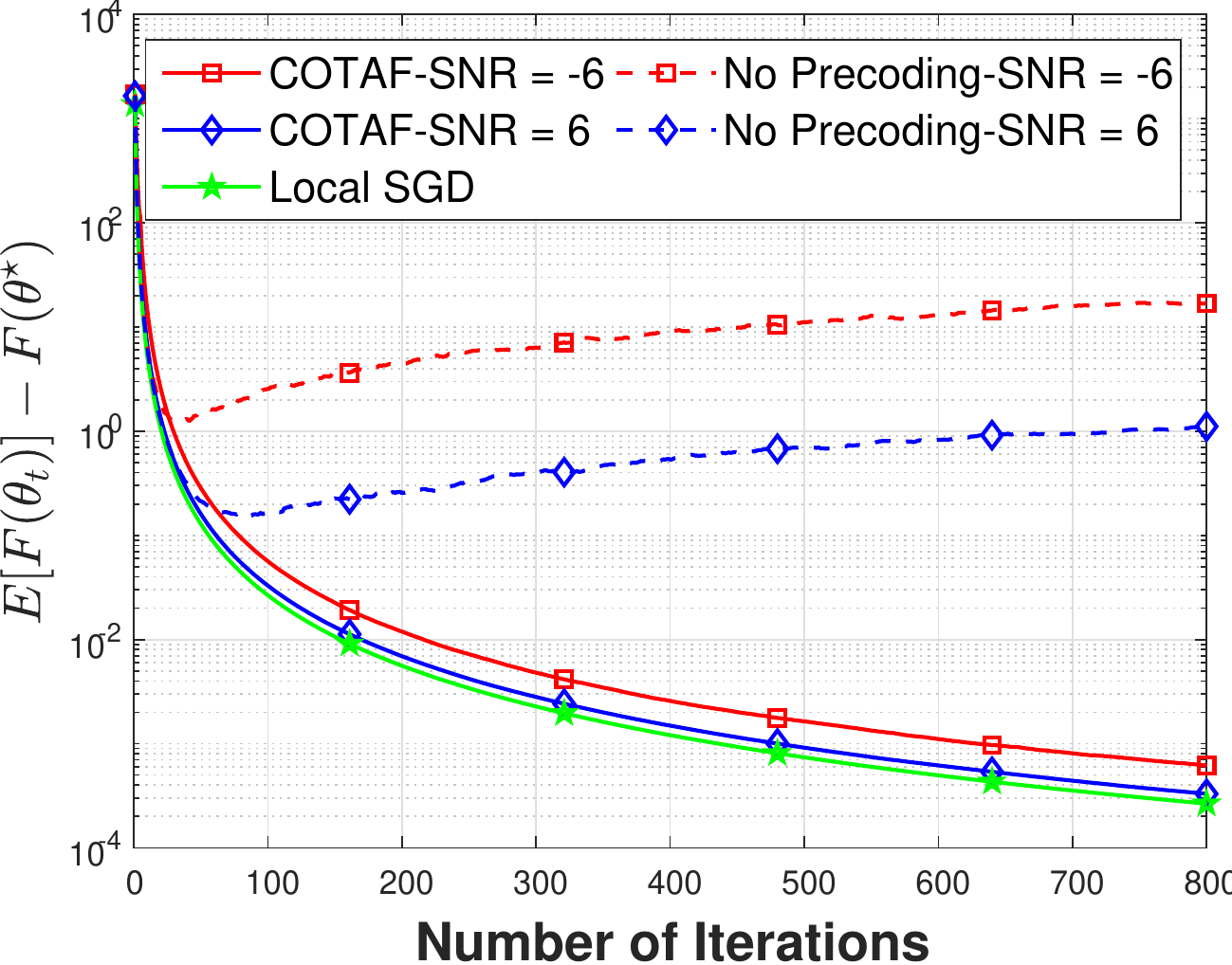}} 
      \caption{Linear predictor, Million Song dataset, $H=40$, $N = 50$, Rayleigh fading channels. 
      }
      \label{fig: COTAF_Year_Pred_fadingChannels}
    \end{center}
\end{figure}

	\vspace{-0.2cm}
\subsection{CNN Classifier Using the CIFAR-10 Dataset}
\label{subsec:cifar}
	\vspace{-0.1cm}
Next, we consider an image classification problem, based on the CIFAR-10 dataset, which contains train and test images from ten different categories. 
The classifier model is the \ac{dnn} architecture detailed in \cite{matlab2020}, which consists of three conventional layers and two fully-connected layers.  When trained in a centralized setting, this architecture achieves an accuracy of roughly $70\%$ \cite{matlab2020}.  
Here, we train this network to minimize the empirical cross-entropy loss in an FL manner, where the data set is distributed among $N=10$ users. Each user holds $5000$ images, and carries out its local training with a minibatch size of 60 images, while aggregation is done every $H=84$ iterations over a MAC with SNR of $-4$ dB.  
We consider two divisions of the training data among the users: {\em i.i.d. data}, where we split the data between the users in an i.i.d fashion, i.e. each user holds the same amount of figures from each class; and {\em heterogeneous data}, where approximately $20\%$ of the training data of each user is associated with a single label, which differs among the different users. This division causes heterogeneity between the users, as each user holds more images from a unique class. 
The model accuracy versus the transmission round achieved for the considered \ac{fl} schemes is depicted in Figs. \ref{fig: COTAF_NN_IID}-\ref{fig: COTAF_NN_nonIID} for the i.i.d. case and the heterogeneous case, respectively.

Observing Figs. \ref{fig: COTAF_NN_IID}-\ref{fig: COTAF_NN_nonIID}, we note that the global model trained using \ac{cotaf} converges to an accuracy of approximately $70\%$, i.e., that of the centralized setting. This is achieved while allowing each user to fully utilize its available temporal and spectral channel resources, thus communicating at higher throughput as compared to orthogonal transmissions.  Furthermore, we point out the following advantages of \ac{cotaf} when applied to CIFAR-10: 
\subsubsection{\ac{cotaf} achieves the desired sublinear convergence rate} While the objective in training the CNN to minimize the cross-entropy loss is not a convex function of the weights, we observe the same rate of convergence for \ac{cotaf} as that of noise-free local SGD. This result suggests a generalization of the theoretical analysis for the convex case and indicates that even in cases in which assumptions \ref{itm:As1} -\ref{itm:As3} do not hold, \ac{cotaf} is still able to converge in a sub-linear rate. We deduce that \ac{cotaf} can be applied in settings less restrictive than the analysed case introduced in Subsection~\ref{subsec:ProtocolPerformance} and still achieve good results, as numerically illustrated in the current study. 

\begin{figure}
    \begin{center}
    {\includegraphics[height=\figHeight,width=\figWidth ]{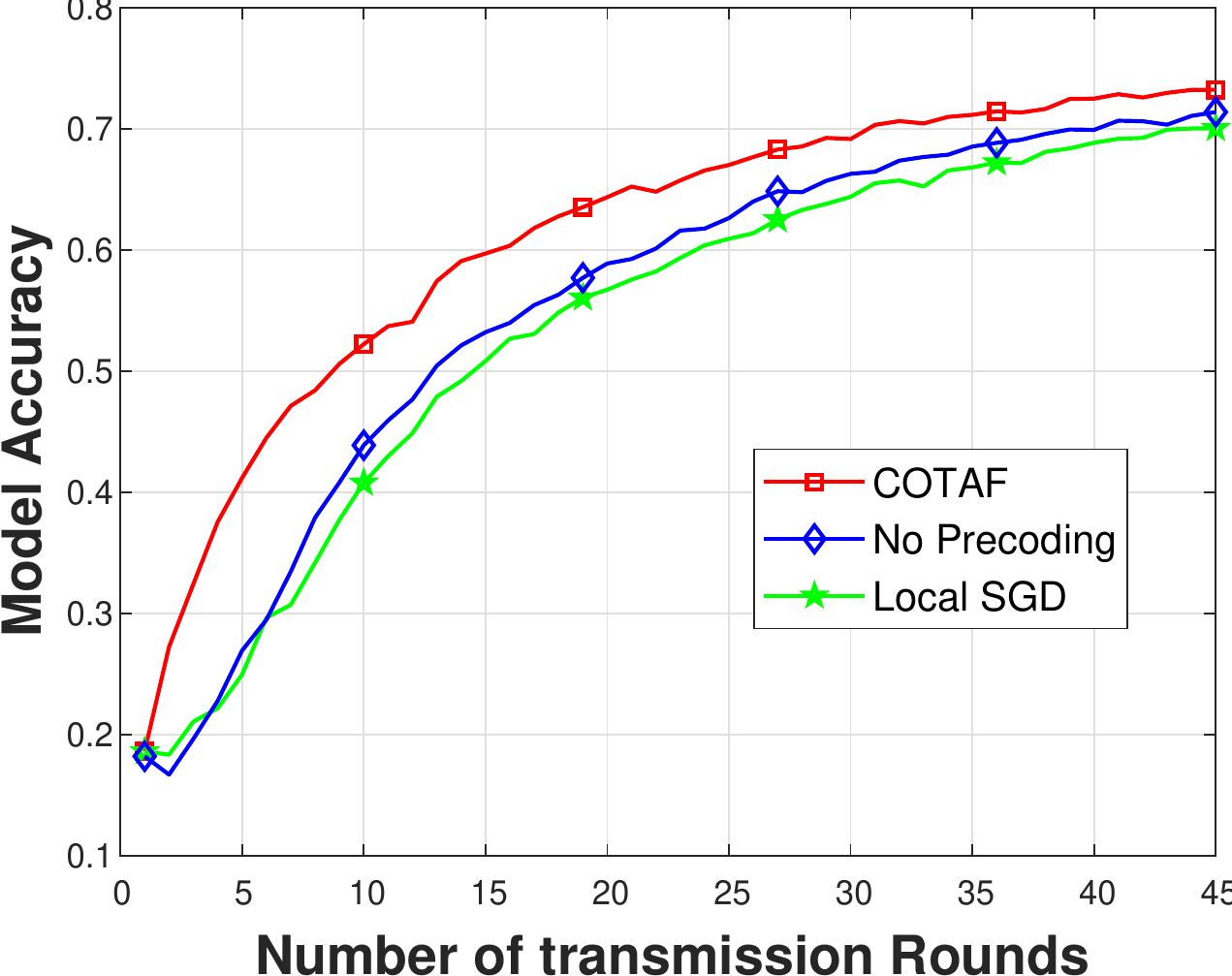}} 
      \caption{CNN,  CIFAR-10 dataset, i.i.d. data.}
      \label{fig: COTAF_NN_IID}
  \vspace{-0.4cm}
    \end{center}
\end{figure}

\subsubsection{\ac{cotaf} benefits from the presence of noise} The simulation results indicate that the additive noise caused by the channel improves the convergence rate and generalization of the CNN model. The fact that \ac{cotaf} gradually mitigates the effective noise allows it to benefit from its presence in non-convex settings, while notably outperforming direct \ac{ota} \ac{fl} with no time-varying precoding operating in the same channel. Specifically, the presence of noise when training \acp{dnn} is known to have positive effects such as reducing overfitting and avoiding local minima. 
Furthermore, we notice that for the heterogeneous data case in Fig. \ref{fig: COTAF_NN_nonIID}, the  gap between \ac{cotaf} and local SGD is increased as compared to the i.i.d case  in Fig. \ref{fig: COTAF_NN_IID}. This indicates that the  noise has a smoothing effect as well. It allows better generalizations in the non-i.i.d setting, which are exploited by \ac{cotaf} in a manner that contributes to its accuracy more effectively as compared to \ac{ota} \ac{fl} with no precoding. 

\begin{figure}
    \begin{center}
    {\includegraphics[height=\figHeight,width=\figWidth ]{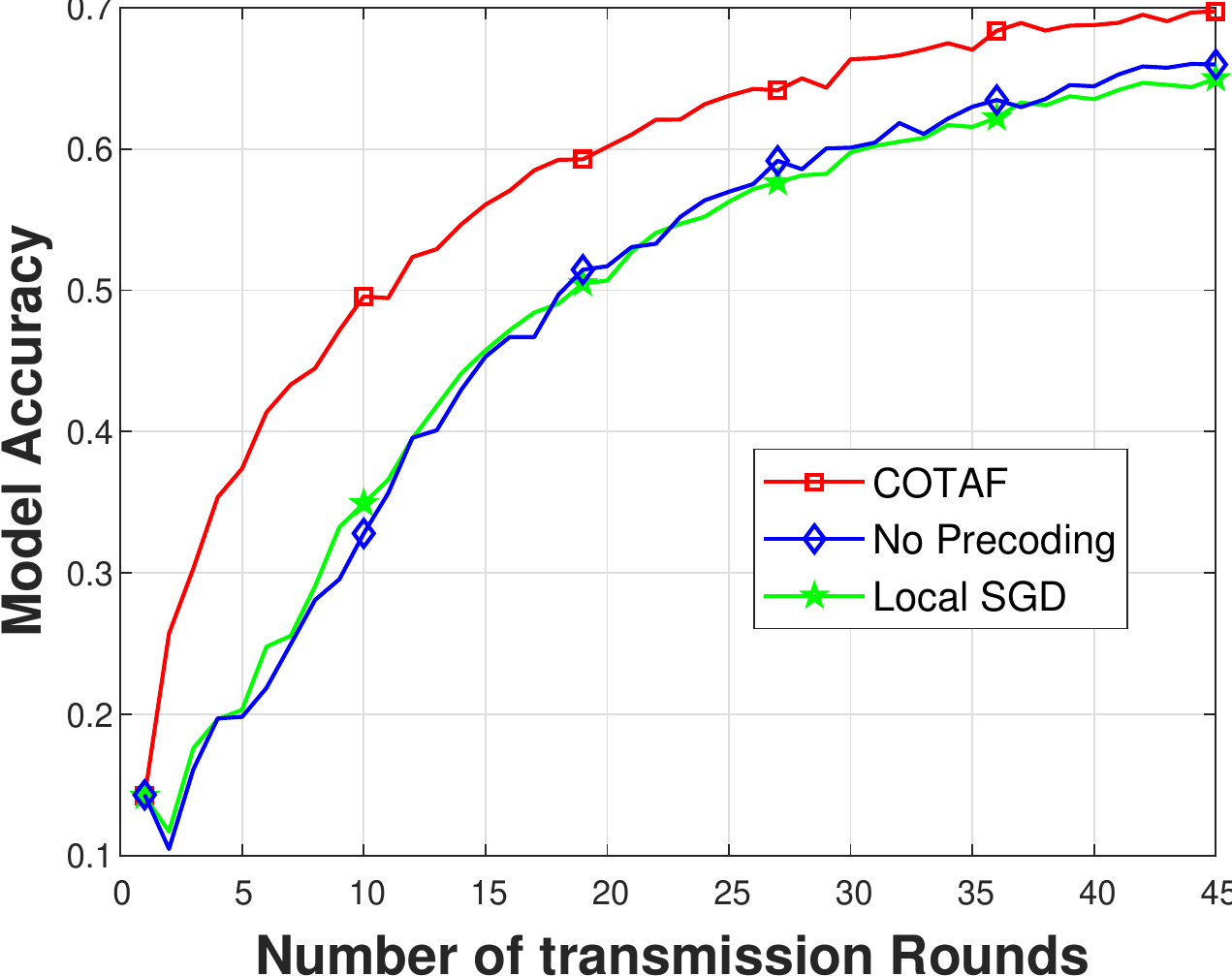}} 
      \caption{CNN,  CIFAR-10 dataset,  heterogeneous data.}
      \label{fig: COTAF_NN_nonIID}
    \end{center}
\end{figure} 

	\vspace{-0.2cm}
	\section{Conclusions}
	\label{sec:Conclusions}
	\vspace{-0.1cm} 	
	In this work we proposed the \ac{cotaf} algorithm for implementing \ac{fl} over wireless \acp{mac}. \ac{cotaf} maintains the convergence properties of local \ac{sgd} with heterogeneous data across users, with convex objectives carried out over ideal channels, without requiring the users to divide the channel resources. This is achieved by introducing a time-varying precoding and scaling scheme which facilitates the aggregation and gradually mitigates the noise effect. We prove that for convex objectives, models trained using \ac{cotaf} with heterogeneous data converge to the loss minimizing model with the same asymptotic convergence rate of local \ac{sgd} over orthogonal channels. Our numerical study demonstrates the ability of \ac{cotaf} to learn accurate models in over wireless channels using non-synthetic datasets. Furthermore, the simulation results show that \ac{cotaf} converges in non-convex settings in a sub-linear rate as well, and outperforms not only \ac{ota} \ac{fl} without precoding, but also the local SGD algorithm in an orthogonal fashion without errors. 
\color{black}

\vspace{-0.2cm}
\begin{appendix}
	\numberwithin{proposition}{subsection} 
	\numberwithin{lemma}{subsection} 
	\numberwithin{corollary}{subsection} 
	\numberwithin{remark}{subsection} 
	\numberwithin{equation}{subsection}	
	%
	\vspace{-0.2cm}
	\subsection{Proof of Theorem \ref{th: theorem 1}}
	\label{app:Proof1}
	In the following, we detail the proof of Theorem \ref{th: theorem 1}, introduced in Subsection \ref{subsec:ProtocolPerformance}. The intermediate derivations detailed below are used in proving Theorem \ref{th: theorem 2} in Appendix \ref{app:Proof2} as well. The outline of the proof is as follows:  First, we define a virtual sequence $\{\thetabar_t\}$ that represents the averaged parameters over all users at every iteration in \eqref{eq: theta^bar definition}, i.e. as if the local \ac{sgd} framework is replaced with mini-batch \ac{sgd}. While $\thetabar_t$ can not be explicitly computed at each time instance by any of the users or the server, it facilitates utilizing bounds established for mini-batch \ac{sgd}, as was done in \cite{stich2018local,li2019convergence}.
	Next, we provide in Lemma \ref{lem : lemma 1} a single step recursive bound for the error $E\left[||\thetabar_t - \thetabf^\star||^2 \right]$. The bound consists four terms, and so, in  Lemmas \ref{lem: lemma 2} and \ref{lem : lemma 3} we upper bound these quantities. Finally, we obtain a non-recursive bound from the recursive expression in Lemma~\ref{lem: lemma 4}, with which we prove Theorem \ref{th: theorem 1}. \\
	\noindent
	\textbf{Recursive error formulation}: Following the steps used in the corresponding convergence analysis of \ac{fl} without communication constraints \cite{stich2018local,li2019convergence}, we first define the virtual sequence $\{\thetabar_t\}_{t\geq 0}$. Broadly speaking,  $\{\thetabar_t\}$ represents the weights obtained when the weights trained by the users are aggregated and averaged over the true channel on each $H$ \ac{sgd} steps, and over a virtual noiseless channel on the remaining \ac{sgd} iterations. This virtual sequence is given by
	\eqspace
	\begin{equation}\label{eq: theta^bar definition}
	    \thetabar_t \triangleq \frac{1}{N} \sum_{n=1}^N \thetabf_t^n + \frac{1}{\sqrt{\alpha_t} N}\myVec{\tilde{w}}_t \mathds{1}_{t\in \mySet{H}}, 
	    \eqspace
	\end{equation}
	where $\mathds{1}_{(\cdot)}$ is the indicator function. 
	Rearranging \eqref{eq: theta^bar definition} to fit our transmission scheme yields:  
	\eqspace
    \begin{equation}\label{eq: theta^bar definition modified} 
          \thetabar_t = \thetabar_{t-H}\!+\! \frac{1}{\sqrt{\alpha_t} N} \sum_{n=1}^N \sqrt{\alpha_t} \left( \thetabf_t^n \!- \!\thetabar_{t-H}\right) \!+\! \myVec{{w}}_t \mathds{1}_{t\in \mySet{H}} , 
          \eqspace
    \end{equation}
    with $\thetabar_t \triangleq \thetabf_0$ for $t \leq 0$. The scaled noise $\myVec{w}_t =\frac{1}{\sqrt{\alpha_t} N}\myVec{\tilde{w}}_t $, and the sequence $\{\thetabf^n_t\}_{t\geq 0}$ are defined in Subsection \ref{subsec:ProtocolDescription}.
    
    Notice that  $\{\thetabar_t\}$ is not computed explicitly, and that  $\thetabar_t = \thetabf_t^n $ for each $n\in \mySet{N}$ whenever $t \in \mySet{H}$. We also define
    \eqspace
    \begin{align}
        \label{eq: g definition} 
            \boldsymbol{g}_t &\triangleq \frac{1}{N}\sum_{n=1}^N \nabla f_{i_t^n}(\thetabf_t^n), \quad 
            \bar{\boldsymbol{g}}_t \triangleq \frac{1}{N}\sum_{n=1}^N \nabla F(\thetabf_t^n). 
            \eqspace
    \end{align} 
    Since the indices $i_t^n$ used in each \ac{sgd} iteration are uniformly distributed, it follows that $\E[\boldsymbol{g}_t] = \bar{\boldsymbol{g}}_t$. 
    By writing  $\myVec{\bar{w}}_t \triangleq \myVec{w}_{t+1}\mathds{1}_{t+1\in\mySet{H}} - \myVec{w}_{t}\mathds{1}_{t\in \mySet{H}}$, we have that
    \eqspace
    \begin{equation}
        \thetabar_{t+1} = \thetabar_t -\eta_t \boldsymbol{g}_t +\myVec{\bar{w}}_t .
        \eqspace
    \end{equation}
   The  equivalent noise vector $\myVec{\bar{w}}_t $ is zero-mean and satisfies
\ifFullVersion
    \begin{align}
        \E[\|\myVec{\bar{w}}_t\|^2] &= \frac{{d}\sigma_w^2}{N^2}\left( \frac{1}{\alpha_{t+1}}\mathds{1}_{t+1\in \mySet{H}} + \frac{1}{\alpha_{t}}\mathds{1}_{t\in \mySet{H}} \right) \nonumber \\
         &\leq \frac{d\sigma_w^2}{N^2 \min(\alpha_{t},\alpha_{t+1})} \mathds{I}_t ,
          \label{eq: w^bar mean & variance}
    \end{align}
\else
\eqspace
    \begin{align}
        \E[\|\myVec{\bar{w}}_t\|^2] 
         &\leq \frac{d\sigma_w^2}{N^2 \min(\alpha_{t},\alpha_{t+1})} \mathds{I}_t ,
          \label{eq: w^bar mean & variance}
          \eqspace
    \end{align}
\fi     
    where $\mathds{I}_t \triangleq  \mathds{1}_{(t\in \mySet{H}) \cup (t+1 \in \mySet{H})}$.
	Theorem \ref{th: theorem 1} is obtained from definitions  \eqref{eq: theta^bar definition} and    \eqref{eq: g definition} via the following lemma:
	\begin{lemma}\label{lem : lemma 1}
	Let $\{\thetabf_t^n\}$ and $\{\thetabar_t\}$ be as defined in \eqref{eq: theta^n_t update} and \eqref{eq: theta^bar definition}, respectively. Then, when \ref{itm:As1}-\ref{itm:As2} are satisfied and the \ac{sgd} step size satisfies $\eta_t \leq \frac{1}{4L}$, it holds that \eqspace
	\begin{align} 
	        &\E\left[||\thetabar_{t+1} - \tstar||^2\right] \leq (1-\mu \eta_t) \E\left[||\thetabar_t-\tstar||^2\right]
	        \notag \\ & \hspace{0.5cm}
	        +\eta_t^2 \E\Big[\big\|\boldsymbol{g}_t - \boldsymbol{\bar{g}}_t + \frac{\boldsymbol{\bar{w}}_t}{\eta_t}\big\|^2\Big]   {-\frac{3}{2}\eta_t \E\left[F(\thetabar_t)  
	        -F^\star \right]}
	        \notag \\ & \hspace{0.5cm}
	        +  \frac{2}{N}\sum_{n=1}^N \E\left[||\thetabar_t-\thetabf_t^n||^2\right] + 6L\eta_t^2 \Gamma .
	    \label{eq: lemma 1}
	    \eqspace
	\end{align}
	\end{lemma}
	\begin{proof}
	    Using the update rule we have:
\ifFullVersion
	    \begin{align} 
	         &||\thetabar_{t+1} - \tstar||^2 = || \thetabar_t- \eta_t \boldsymbol{g}_t - \tstar+ \myVec{\bar{w}}_t||^2 
	         \notag \\ 
	      &\quad = || \thetabar_t- \eta_t \boldsymbol{g}_t - \tstar - \eta_t \boldsymbol{\bar{g}}_t +\eta_t \boldsymbol{\bar{g}}_t  + \myVec{\bar{w}}_t ||^2
	        \notag\\  
	         &\quad= || \thetabar_t- \eta_t \boldsymbol{\bar{g}}_t - \tstar||^2 + \eta_t^2 \left\| \boldsymbol{\bar{g}}_t - \boldsymbol{g}_t +\frac{\myVec{\bar{w}}_t}{\eta_t}\right\|^2
	        \notag\\ 
	        &\qquad +2\eta_t \left<\thetabar_t-\tstar - \eta_t \bar{\boldsymbol{g}_t}  ,\bar{\boldsymbol{g}}_t-\boldsymbol{g}_t  
	        +\frac{\myVec{\bar{w}}_t}{\eta_t} \right>.  
	        \label{eqn:proofAid1}
	    \end{align}
\else
	    \eqspace
	    \begin{align} 
	         &||\thetabar_{t+1} - \tstar||^2 =   \| \thetabar_t- \eta_t \boldsymbol{\bar{g}}_t - \tstar\|^2 + \eta_t^2 \left\| \boldsymbol{\bar{g}}_t - \boldsymbol{g}_t +\frac{\myVec{\bar{w}}_t}{\eta_t}\right\|^2
	        \notag\\ 
	        &\qquad +2\eta_t \left<\thetabar_t-\tstar - \eta_t \bar{\boldsymbol{g}_t}  ,\bar{\boldsymbol{g}}_t-\boldsymbol{g}_t  
	        +\frac{\myVec{\bar{w}}_t}{\eta_t} \right>.  
	        \label{eqn:proofAid1}
	    \eqspace
	    \end{align}
\fi
	    Observe that $E\big[ \big<\thetabar_t-\tstar - \eta_t \bar{\boldsymbol{g}_t}  ,\bar{\boldsymbol{g}}_t-\boldsymbol{g}_t  
	        +\frac{\myVec{\bar{w}}_t}{\eta_t} \big>\big] = 0 $. Following the proof steps in \cite[Lemma 1]{li2019convergence}, we obtain
\ifFullVersion	        
	        \begin{align*} 
	            &|| \thetabar_t\! - \! \eta_t \boldsymbol{\bar{g}}_t \! - \! \tstar||^2 \leq  (1\! - \!\mu \eta_t) ||\thetabar_t \! - \! \tstar||^2\! +\!\frac{1}{N}\sum_{n=1}^N ||\thetabar_t\! - \!\thetabf_n^t ||^2 \notag \\ &\!+\! \underbrace{  \frac{4L\eta_t^2}{N}\!\sum_{n=1}^N\!(f_n(\thetabf_t^n) \! - \!f_n^\star) \! - \!  \frac{2\eta_t}{N}\!\sum_{n=1}^N \!f_n(\thetabf_t^n) \! - \!f_n(\tstar). }_{\triangleq A} 
	        \end{align*}
\else
$|| \thetabar_t\! - \! \eta_t \boldsymbol{\bar{g}}_t \! - \! \tstar||^2 \leq  (1\! - \!\mu \eta_t) ||\thetabar_t \! - \! \tstar||^2\! +\!\frac{1}{N}\sum_{n=1}^N ||\thetabar_t\! - \!\thetabf_n^t ||^2\!+\! A$, with $A \triangleq \frac{4L\eta_t^2}{N}\!\sum_{n=1}^N\!(f_n(\thetabf_t^n) \! - \!f_n^\star) \! - \!  \frac{2\eta_t}{N}\!\sum_{n=1}^N \!f_n(\thetabf_t^n) \! - \!f_n(\tstar)$.
\fi
        Observe that $A$ defined above satisfies:
\ifFullVersion
        \begin{align}
            A &= \frac{1}{N} \sum_{n=1}^N \left[ (4L\eta_t^2 \! - \!2\eta_t)f_n(\thetabf_t^n)\! - \! 4L\eta_t^2 f_n^\star +2 \eta_t f_n(\tstar)  \right] \nonumber \\
            & \stackrel{(a)}{=}   \frac{4L\eta_t^2\!  -\! 2\eta_t}{N}\sum_{n=1}^N  (f_n(\thetabf_t^n)\! -\! F^\star ) 
           \!  + \!  \frac{4L\eta_t^2}{N}\sum_{n=1}^N (F^\star \! -\!  f_n^\star)  \nonumber\\
            &\stackrel{(b)}{=} (4L\eta_t^2 -2\eta_t)  \frac{1}{N}\sum_{n=1}^N (f_n(\thetabf_t^n)-F^\star ) + 4L\eta_t^2 \Gamma,
            \label{eq: A bound, part 1}
        \end{align}
        where $(a)$ and $(b)$ follow from the definitions of $F^\star$ and $\Gamma$, respectively.
  \else
	    \eqspace
         \begin{align}
            A &= \frac{1}{N} \sum_{n=1}^N \left[ (4L\eta_t^2 \! - \!2\eta_t)f_n(\thetabf_t^n)\! - \! 4L\eta_t^2 f_n^\star +2 \eta_t f_n(\tstar)  \right] \nonumber \\ 
            &\stackrel{(a)}{=} (4L\eta_t^2 -2\eta_t)  \frac{1}{N}\sum_{n=1}^N (f_n(\thetabf_t^n)-F^\star ) + 4L\eta_t^2 \Gamma,
            \label{eq: A bound, part 1}
	    \eqspace
        \end{align}
        where $(a)$   follows from the definitions of $F^\star$ and $\Gamma$. 
\fi        
        Notice that $4L\eta_t^2 -2\eta_t \leq \eta_t  -2\eta_t \leq -\eta_t$.\\ \indent
        Next, we use the following inequality obtained in \cite{li2019convergence}
\ifFullVersion        
        \begin{align} 
           &\frac{1}{N}\sum_{n=1}^N (f_n(\thetabf_t^n)-F^\star) \geq (F(\thetabar_t) - F^\star ) \notag \\
           &-\frac{1}{N}\sum_{n=1}^N \left[ \eta_t L(f_n (\thetabar_t)\!-\! f_n^\star)\! +\!\frac{1}{2\eta_t}||\thetabf_t^n\!-\! \thetabar_t||^2 \right].
           \label{eq: obtaining theta^bar}
        \end{align}
        Substituting \eqref{eq: obtaining theta^bar} 
\else
$\frac{1}{N}\sum_{n=1}^N (f_n(\thetabf_t^n)-F^\star) \geq (F(\thetabar_t) - F^\star ) -\frac{1}{N}\sum_{n=1}^N \big[ \eta_t L(f_n (\thetabar_t)\!-\! f_n^\star)\! +\!\frac{1}{2\eta_t}||\thetabf_t^n\!-\! \thetabar_t||^2 \big]$. Substituting this 
\fi
        into \eqref{eq: A bound, part 1}
  \ifFullVersion      
        yields:
        \begin{align}
            &A  \leq \frac{2\eta_t \! - \!4L\eta_t^2 }{N}\sum_{n=1}^N \left[ \eta_t L(f_n (\thetabar_t)\! - \! f_n^\star) \!+\!\frac{1}{2\eta_t}||\thetabf_t^n\! - \! \thetabar_t||^2 \right] \notag \\
            &\quad \! - \! (2\eta_t \! - \!4L\eta_t^2) (F(\thetabar_t) \! - \! F^\star ) + 4L\eta_t^2 \Gamma \nonumber \\
            & = ( 2\eta_t \! - \!4L\eta_t^2 ) (\eta_t L \! - \!1 ) (F(\thetabar_t) \! - \! F^\star ) \notag \\
             &+ \! ( (2\eta_t \! - \! 4L\eta_t^2)\eta_t \!+\! 4\eta_t^2)L\Gamma  
             \!+\! \frac{(2\eta_t \! - \!4L\eta_t^2)}{2\eta_t N}  \sum_{n=1}^N||\thetabf_t^n\! - \! \thetabar_t||^2 \nonumber \\
            & \stackrel{(a)}{\leq} \! 6L \eta_t^2 \Gamma \! - \! \frac{3\eta_t}{2}(F(\thetabar_t) \! - \! F^\star ) \!+ \!\frac{1}{N}\!\sum_{n=1}^N\!||\thetabf_t^n\! - \! \thetabar_t||^2,
            \label{eq: A bound, part 2}
        \end{align}
        where in $(a)$ we use the following facts: (1) $\eta_t L -1 \leq -\frac{3}{4}$, (2) $ 2\eta_t -4L\eta_t^2 \leq 2\eta_t$.
\else
    and using the fact that  $\eta_t L -1 \leq -\frac{3}{4}$, (2) $ 2\eta_t -4L\eta_t^2 \leq 2\eta_t$ yields
	    \eqspace
        \begin{align}
            \!\!A\!  \leq   \! 6L \eta_t^2 \Gamma \! - \! \frac{3\eta_t}{2}(F(\thetabar_t) \! - \! F^\star ) \!+ \!\frac{1}{N}\!\sum_{n=1}^N\!||\thetabf_t^n\! - \! \thetabar_t||^2.
            \label{eq: A bound, part 2}
	    \eqspace
        \end{align}
\fi
        Consequently, we have that:
	    \eqspace
        \begin{align}
           & ||\thetabar_t -\tstar- \eta_t \bar{\myVec{g}}_t||^2 \leq (1-\mu \eta_t)  ||\thetabar_t -\tstar||^2 \notag \\ &+\! \frac{2}{N}\sum_{n=1}^N ||\thetabar_t\!-\!\thetabf_t^n||^2 
            \!+\! 6L \eta_t^2 \Gamma \!-\! \frac{3\eta_t}{2}(F(\thetabar_t) \!-\! F^\star ).
             \label{eq: lemma1 almost done}
	    \eqspace
        \end{align}
        Finally, by taking the expected value of both sides of \eqref{eqn:proofAid1} and using \eqref{eq: lemma1 almost done} we complete the proof.
	    \end{proof}	

\smallskip
	\noindent
    \textbf{Upper bounds on the additive terms}:
	Next, we prove the theorem by bounding the summands constituting the right hand side of \eqref{eq: lemma 1}. First, we bound $\E\big[||\boldsymbol{g}_t - \boldsymbol{\bar{g}}_t +\frac{\myVec{\bar{w}}_t}{\eta_t}||^2\big] $, as stated in the following lemma:
	\begin{lemma}\label{lem: lemma 2}
    	   When the step size sequence $\{\eta_t\}$ consists of decreasing positive numbers satisfying $\eta_t \leq 2\eta_{t+H}$ for all $t\geq 0$ and \ref{itm:As3} holds, then
	    \eqspace
	    \begin{align*} 
	      &\E\big[\big\|\boldsymbol{g}_t - \boldsymbol{\bar{g}}_t +\frac{\myVec{\bar{w}}_t}{\eta_t} \big\|^2\big] 
	      \leq \frac{1}{N^2}\sum_{n=1}^N M_n^2   +\frac{4 d H^2 G^2\sigma^2_w}{P N^2}  \mathds{I}_t . 
	    \eqspace
	    \end{align*}
	\end{lemma}
	\begin{proof} The lemma follows since the noise term $\myVec{\bar{w}}_t$ is zero-mean and independent of the stochastic gradients, hence
\ifFullVersion	
	    \begin{align}  
	     & \E\left[\left\|\bar{\boldsymbol{g}}_t-\boldsymbol{g}_t   +\frac{\myVec{\bar{w}}_t}{\eta_t}\right\|^2 \right] = \E[||\bar{\boldsymbol{g}}_t-\boldsymbol{g}_t ||^2] + \E\left[\left\|\frac{\myVec{\bar{w}}_t}{\eta_t}\right\|^2\right] 
	       \notag \\ &\qquad\qquad
	      \stackrel{(a)}{\leq} \frac{1}{N^2}\sum_{n=1}^N M_n^2+\E\left[\left\|\frac{\myVec{\bar{w}}_t}{\eta_t}\right\|^2\right],
	      \label{eq: lemma 2 proof} 
	    \end{align}
	    where $(a)$ follows from \ref{itm:As3}. The first summand in \eqref{eq: lemma 2 proof}  coincides with \cite[Lemma 2]{li2019convergence}.
\else
	    $\E\big[\big\|\bar{\boldsymbol{g}}_t \!-\!\boldsymbol{g}_t   \!+\!\frac{\myVec{\bar{w}}_t}{\eta_t}\big\|^2 \big] \!\leq\! \frac{1}{N^2}\!\sum_{n=1}^N M_n^2\!+\!\E\big[\big\|\frac{\myVec{\bar{w}}_t}{\eta_t}\big\|^2\big]$, 
	    by \ref{itm:As3}. 
\fi
	 From \eqref{eq: w^bar mean & variance} we obtain
	    \eqspace
	\begin{align} \label{eq: w^bar ineq1}
	   \E\Big[\big\|\frac{\myVec{\bar{w}}_t}{\eta_t}\big\|^2\Big] & \!\leq\! \frac{d\sigma_w^2 \mathds{I}_t }{\eta_t^2 N^2\min(\alpha_{t},\alpha_{t+1})}.
	    \eqspace
	\end{align}
    Next, we bound $\frac{1}{\alpha_t} = \frac{1}{P}\max_n \E\left[ ||\thetabf_t^n - \thetabf_{t-H}^n ||^2 \right]$ via:
	    \eqspace
    \begin{align} 
            \frac{1}{\alpha_t} 
            &
            \stackrel{(a)}{\leq} \frac{1}{P} \max_n \Big( H \eta_{t-H} \sum_{t'=t-H}^ {t-1} \E\left[||\nabla f_{i_h^k} (\thetabar_{t'}^n)||^2\right]  \Big)  \notag\\ &
           \stackrel{(b)}{\leq} \frac{1}{P}H^2 \eta_{t-H}^2 G^2 
         \stackrel{(c)}{\leq}\frac{1}{P} 4H^2 \eta_t^2 G^2,
        \label{eq: 1/alpha_t bound}
	    \eqspace
    \end{align}
    where $(a)$ follows from \eqref{eq: SGD}, using the  inequality $\|\sum_{t'={t-H}}^{t} \myVec{r}_t\|^2 \leq H \sum_{t'={t-H}}^{t}  \|\myVec{r}_t\|^2$, which holds for any multivariate sequence $\{\myVec{r}_t\}$, while noting that the step size is monotonically non-increasing; $(b)$ holds by \ref{itm:As3}; and $(c)$ holds as  $\eta_t \leq 2\eta_{t+H}$ for all $t\geq 0$. Finally, notice that
\ifFullVersion
    \begin{align}
       & \frac{1}{\min(\alpha_{t},\alpha_{t+1})} = \max\left(\frac{1}{\alpha_{t}},\frac{1}{\alpha_{t+1})}\right)\notag \\
        & \qquad \leq \frac{1}{P}4H^2 G^2\max( \eta_t^2 , \eta_{t+1}^2 ) 
        = \frac{1}{P} 4H^2 \eta_t^2 G^2,
        \label{eq: bound3}
    \end{align}
\else
	    \eqspace
    \begin{align}
       & \frac{1}{\min(\alpha_{t},\alpha_{t+1})}  \leq  \frac{1}{P} 4H^2 \eta_t^2 G^2,
        \label{eq: bound3}
	    \eqspace
    \end{align}
\fi
    as $\{\eta_t\}$ is monotonically decreasing.
    Substituting \eqref{eq: bound3} into \eqref{eq: w^bar ineq1} completes the proof.
	\end{proof} 
	\noindent
	In the next lemma, we bound $\frac{1}{N}\sum_{n=1}^N \E\left[||\thetabar_t-\thetabf_t^n||^2\right]$:
	\begin{lemma} \label{lem : lemma 3}
   	   When the step size sequence $\{\eta_t\}$ consists of decreasing positive numbers satisfying $\eta_t \leq 2\eta_{t+H}$ for all $t\geq 0$ and \ref{itm:As3} holds, then
	    $\E\left[||\thetabar_t-\thetabf_t^n||^2\right] \leq 4\eta_t^2 G^2 H^2$.
	\end{lemma}
		\begin{proof} 
		The lemma follows directly from \cite[Lem. 3.3]{stich2018local}.
			\end{proof} 
	\smallskip
	\noindent
	\textbf{Obtaining a non-recursive convergence bound}:
	Combining Lemmas \ref{lem : lemma 1}-\ref{lem : lemma 3} yields a recursive relationship which allows us to characterize the convergence of \ac{cotaf}. To complete the proof, we next establish the convergence bound from the recursive equations,  based on the following lemma:
	\begin{lemma} \label{lem: lemma 4}
	    Let $\{\delta_t\}_{t\geq 0}$ and $\{e_t\}_{t\geq 0}$ be two positive sequences satisfying 
	    \eqspace
    \begin{equation}\label{eq: lemma 4.1 inequality}
        \delta_{t+1} \leq (1-\mu \eta_t)\delta_t - \eta_t e_t A + \eta_t^2 B  + \eta_t^2 D_t,
	    \eqspace
    \end{equation}
    for $\eta_t = \frac{4}{\mu(a+t)}$  with constants $A>0 , B,C \geq 0, \mu>0, a>1$ and $D_t = D \mathds{1}_{t\in \mySet{H}}$.   Then, for any positive integer $R$   it holds that
\ifFullVersion    
    \begin{align}
            \frac{A}{S_R} \sum_{r=1}^{R} \beta_{rH} e_{rH}\leq & \frac{\mu a^3}{4S_R}\delta_0 + \frac{2R(H+1)}{\mu S_R}B (2a+H+R-1) 
            \notag \\ 
            & +\frac{2R}{\mu S_R}D(2a+ HR +H),
            \label{eq: lemma 4.1 inequality2}
    \end{align}
        for $ \beta_t= (a+t)^2$, $T=RH$ and  $S_R = \sum_{r=1}^R \beta_{rH} = a^2 R +aT(R+1) +\frac{2T^2 R+ T^2 +TH}{6} \geq \frac{1}{3H}T^3 = \frac{1}{3}T^2 R$.
\else
  $\frac{A}{S_R} \sum_{r=1}^{R} \beta_{rH} e_{rH}\leq   \frac{\mu a^3}{4S_R}\delta_0 + \frac{2R(H+1)}{\mu S_R}B (2a+H+R-1)   +\frac{2R}{\mu S_R}D(2a+ HR +H)$,     for $ \beta_t= (a+t)^2$, $T=RH$ and  $S_R = \sum_{r=1}^R \beta_{rH}  \geq \frac{1}{3H}T^3 = \frac{1}{3}T^2 R$.
\fi

	\end{lemma}
	\begin{proof}
	    To prove the lemma, we first note that by \cite[Eqn.(45)]{Stich2018sparsified}, it holds that
\ifFullVersion	    
        \begin{align}
            \left( 1- \mu \eta_t\right)\frac{\beta_t}{\eta_t} 
            &= \left(\frac{a+t-4}{4}\right)\frac{\mu (a+t)^3}{a+t} \notag \\
            &= \frac{\mu(a+t-4)(a+t)^2}{4}\notag \\ 
            &\leq \frac{\mu (a+t-1)^3}{4}= \frac{\beta_{t-1}}{\eta_{t-1}}.
            \label{eq: lemma 4 proof ineq1} 
        \end{align} 
\else 
$ \left( 1- \mu \eta_t\right)\frac{\beta_t}{\eta_t} \leq \frac{\mu (a+t-1)^3}{4}= \frac{\beta_{t-1}}{\eta_{t-1}}$.
\fi
        Therefore, multiplying \eqref{eq: lemma 4.1 inequality} by $\frac{\beta_t}{\eta_t}$  yields
        \begin{align}
            \delta_{t\!+\!1} \frac{\beta_t}{\eta_t}\!\leq& (1\!-\!\mu \eta_t) \frac{\beta_t}{\eta_t}\delta_t\!-\!\beta_t e_t A \!+ \!\beta_t \eta_t B
            \!+ \!\beta_t\eta_t D_t.
            \label{eq: lemma 4 proof ineq2}
        \end{align}
        
        Next, we extract the relations between two sequential rounds, i.e. $\delta_t, \delta_{t+H}$. Note that in each round a single $D_t$ is activated, i.e., only a single entry in the set $\{D_\tau\}_{\tau=t}^{t+H}$ is non-zero. Therefore, by repeating the recursion \eqref{eq: lemma 4 proof ineq2} over $H$ time instances, we get
\ifFullVersion        
        \begin{align} \label{eq: bound over H iterations}
            \delta_{t+H} \frac{\beta_{t+H-1}}{\eta_{t+H-1}}&\leq (1-\mu \eta_t) \frac{\beta_t}{\eta_t}\delta_t -\sum_{\tau=t}^{t+H-1} \beta_t e_t A \notag \\
           &\quad + \sum_{\tau=t}^{t+H-1} \beta_t \eta_t B + \beta_{t_0}\eta_{t_0} D \notag \\  &\leq (1-\mu \eta_t) \frac{\beta_t}{\eta_t}\delta_t\!+\! \frac{2B}{\mu}(H\!+\!1)(H\!+\!2t\!+\!2a) \notag \\ & \quad  - \beta_t e_t A   + \beta_{t_0}\eta_{t_0} D ,
        \end{align}
\else
        \begin{align} 
            \delta_{t+H} \frac{\beta_{t+H-1}}{\eta_{t+H-1}}
           &\leq (1-\mu \eta_t) \frac{\beta_t}{\eta_t}\delta_t\!+\! \frac{2B}{\mu}(H\!+\!1)(H\!+\!2t\!+\!2a) \notag \\ & \quad  - \beta_t e_t A   + \beta_{t_0}\eta_{t_0} D ,
           \label{eq: bound over H iterations}
        \end{align}
\fi
        where $ t_0$ is the only time instance in the interval $[t,t+H)$ such that $ t_0 \in \mySet{H}$. In the last inequality we used the fact that $A \beta_t e_t >0$.
        Recursively applying \eqref{eq: bound over H iterations} $R$ times yields:
\ifFullVersion        
        \begin{align}
            \delta_R \frac{\beta_R}{\eta_R} \leq & (1-\mu \eta_0) \frac{\beta_0}{\eta_0}\delta_0 -\sum_{r=1}^R \beta_{rH} e_{rH} A + \sum_{r=1}^R \beta_{rH}\eta_{rH} D  \notag \\ &+ \sum_{r=1}^R\frac{2B}{\mu}(H+1)(H+2rH+2a). 
            \label{eqn: bound4}
        \end{align}
        As $\delta_t \frac{\beta_t}{\eta_t} > 0$ for each $t$, \eqref{eqn: bound4} 
                implies that
        \begin{align}
          A\sum_{r=1}^R \beta_{rH} e_{rH} \leq &\frac{\beta_0}{\eta_0} \delta_0 +\frac{2B(H+1)}{\mu} \sum_{r=1}^R(H+2rH+2a)  \notag\\&  +\sum_{r=1}^R \beta_{rH}\eta_{rH} D. \label{eq: lemma 4 sums}
        \end{align}
\else
$\delta_R \frac{\beta_R}{\eta_R} \leq  (1-\mu \eta_0) \frac{\beta_0}{\eta_0}\delta_0 -\sum_{r=1}^R \beta_{rH} e_{rH} A + \sum_{r=1}^R \beta_{rH}\eta_{rH} D  + \sum_{r=1}^R\frac{2B}{\mu}(H+1)(H+2rH+2a)$. 
        As $\delta_t \frac{\beta_t}{\eta_t} > 0$ for each $t$, this         implies that
        $  A\sum_{r=1}^R \beta_{rH} e_{rH} \leq \frac{\beta_0}{\eta_0} \delta_0 +\frac{2B(H+1)}{\mu} \sum_{r=1}^R(H+2rH+2a)   +\sum_{r=1}^R \beta_{rH}\eta_{rH} D$.
\fi

        Next, recalling that $\beta_t\eta_t = \frac{4(a+t)}{\mu}$, we obtain:
        \begin{align}
            \sum_{r=1}^R \beta_{rH}\eta_{rH} D
            &= \frac{4}{\mu} D \left( aR + \frac{HR}{2} + \frac{H R^2}{2} \right). \label{eq: lemma 4 D_t sum}  
               \end{align}
        For the current setting of $\beta_t$ and $\StepSize_t$ it holds that $\frac{\beta_0}{\eta_0} = \frac{\mu a^3}{4}$. Further, $ \sum_{r=1}^R(H+2rH+2a) = R(2a+T +2H)$. Substituting this and  \eqref{eq: lemma 4 D_t sum}  into the above inequality proves the lemma.
	\end{proof} 

	We complete the proof of the theorem by combining Lemmas \ref{lem : lemma 1}-\ref{lem: lemma 4} as follows: By defining  $\delta_t \triangleq \E\big[||\thetabar_{t} - \tstar||^2\big]$, it follows from  Lemma \ref{lem : lemma 1} combined with the bounds stated in Lemmas \ref{lem: lemma 2}-\ref{lem : lemma 3} that:
	\begin{align} \label{eq: lems combined inequal}
	    \delta_{t+1} &\leq (1-\mu \eta_t) \delta_t + \eta_t^2 \left( \frac{1}{N^2} \sum_{n=1}^N M_n^2 \right. + \left. \frac{4d H^2 G^2\sigma_w^2}{P N^2}\mathds{I}_t  \right)\notag \\ 
	    & -\frac{3}{2}\StepSize_t \E[F(\thetabar_t) - F^\star] + 8\StepSize_t^2 H^2 G^2 +6L\eta_t^2 \Gamma .
	\end{align}
	In the non-trivial case where $H>1$, at most one element of  $\{t_0+1,t_0\}$ can be in $\mySet{H}$ for any $t_0$. Therefore, without loss of generality, we  reduce the set over which the indicator function in \eqref{eq: lems combined inequal} is defined to be
	$\{t\in \mySet{H}\}$.
	By defining
	\begin{equation*}
	\begin{array}{l}
	    A \triangleq \frac{3}{2} ;\quad B \triangleq 8H^2 G^2 + \frac{1}{N^2}\sum_{n=1}^N M_n^2 +6L\Gamma;
	    \vspace{0.1cm}\\ \displaystyle \hspace{0.1cm}
	    e_t \triangleq \E\left[F(\thetabar_t) - F^\star\right]; \quad D_t \triangleq \frac{4d H^2 G^2\sigma_w^2}{PN^2}\mathds{1}_{t\in \mySet{H}},
	    \end{array}
	\end{equation*}
	and plugging these notations into Lemma \ref{lem: lemma 4}, we obtain   
\ifFullVersion	
	\begin{align*}
	   & \frac{1}{S_R}\sum_{r=1}^{R} \beta_{r H} \E\left[F(\thetabar_{rH}) - F^\star\right] \leq \frac{\mu a^3}{6S_R}||\thetabf_0 -\tstar||^2 \notag \\ 
        &+\!\frac{4(T\!+ \!R)}{3\mu S_R}(2a \!+ \! H\! + \!R\!- \!1)B 
        \!+ \! \frac{16d TH G^2\sigma_w^2}{3\mu P N^2 S_R} (2a\!+ \!T\!+ \!H).
	\end{align*}
	Finally, by the convexity of the objective function, it holds that 
	\begin{equation}
	 \E[F(\hat{\thetabf}_T) - F^\star] \leq \frac{1}{S_R}\sum_{r=1}^R \beta_{rH} \E[F\left(\thetabar_{rH}\right) - F^\star],
	\end{equation}
\else
$\frac{1}{S_R}\sum_{r=1}^{R} \beta_{r H} \E\left[F(\thetabar_{rH}) - F^\star\right] \leq \frac{\mu a^3}{6S_R}||\thetabf_0 -\tstar||^2    +\!\frac{4(T\!+ \!R)}{3\mu S_R}(2a \!+ \! H\! + \!R\!- \!1)B 
        \!+ \! \frac{16d TH G^2\sigma_w^2}{3\mu P N^2 S_R} (2a\!+ \!T\!+ \!H)$. 
	Finally, by the convexity of the objective function, it holds that 
	$ \E[F(\hat{\thetabf}_T) - F^\star] \leq \frac{1}{S_R}\sum_{r=1}^R \beta_{rH} \E[F\left(\thetabar_{rH}\right) - F^\star]$, 
\fi
	\color{black}
	thus proving Theorem \ref{th: theorem 1}.
	\qed

	\vspace{-0.2cm}
	\subsection{Proof of Theorem \ref{th: theorem 2}}
	\label{app:Proof2}
	The proof of Theorem \ref{th: theorem 2} utilizes Lemmas \ref{lem : lemma 1}-\ref{lem : lemma 3}, stated in Appendix \ref{app:Proof1}, while formulating an alternative non-recursive bound compared to that used in Appendix \ref{app:Proof1}. To obtain the convergence bound in \eqref{eq: theorem 2}, we first recall the definition  $\delta_t \triangleq  \E\left\{\left\| \thetabar_{t+1} - \myWeights\Opt\right\|^2 \right\} $. When $t \in \mySet{H}$, the term $\delta_t $ represents the $\ell_2$ norm of the error in the weights of the global model. We can upper bound \eqref{eq: lems combined inequal} and formulate the following recursive relationship on the weights error
\begin{equation}
\delta_{t+1} \leq (1- \StepSize_{t}\ConvParam) \delta_t + \StepSize_{t}^2 C,
\label{eqn:RecursiveRel}
\end{equation}
where $C =  B + \frac{4d H^2 G^2 \sigma_w^2}{PN^2}$. The inequality is obtained from \eqref{eq: lems combined inequal} since  $-\eta_t e_t \E\left[F(\thetabar_t) - F\Opt\right] \leq 0 $ and as $D\mathds{1}_{t\in \mySet{H}} \leq D$, for $D \geq 0$. 
The convergence bound is achieved by properly setting the step-size and the \ac{fl} systems parameters in \eqref{eqn:RecursiveRel} to bound $\delta_{t}$, and combining the resulting bound with the strong convexity of the objective. 
In particular, we set the step size $\StepSize_{t}$ to take the form $\StepSize_{t} = \frac{\StepSizeNume}{t + \gamma}$ for some $\StepSizeNume > \frac{1}{\ConvParam}$ and $\gamma \geq \max\big(4 \SmoothParam \StepSizeNume, \SGDIter\big)$, for which $\StepSize_{t} \leq \frac{1}{4\SmoothParam}$ and $\StepSize_{t} \leq 2\StepSize_{t+\SGDIter}$, implying that Lemmas \ref{lem: lemma 2}-\ref{lem : lemma 3} hold.

Under such settings, we show that there exists a finite $\nu$ such that $\delta_{t}  \leq \frac{\nu}{t + \gamma}$ for all integer $l \geq 0$. We prove this by induction, noting that setting $\nu \geq\gamma \delta_{0}$ guarantees that it holds for $t=0$. We next show that if  $\delta_{t}  \leq \frac{\nu}{t + \gamma}$, then  $\delta_{t+1}  \leq \frac{\nu}{t+1 + \gamma}$. It follows from  \eqref{eqn:RecursiveRel} that 
\ifFullVersion
\begin{align}
\delta_{t+1} &\leq \left(1- \frac{\StepSizeNume}{ t + \gamma}\ConvParam\right)  \frac{\nu}{t + \gamma}+ \left( \frac{\StepSizeNume}{ t+ \gamma}\right)^2  C  \notag \\
&=   \frac{1}{t+\gamma}\left(\left(1- \frac{\StepSizeNume}{ t + \gamma}\ConvParam\right)\nu + \frac{\StepSizeNume^2}{ t+ \gamma}  C \right) . 
\label{eqn:ConvProof1}
\end{align}
Consequently,  $\delta_{t+1}  \leq \frac{\nu}{t+1 + \gamma}$ holds when  
\begin{equation*}
\frac{1}{t+\gamma}\left(\left(1- \frac{\StepSizeNume}{ t + \gamma}\ConvParam\right)\nu + \frac{\StepSizeNume^2}{ t+ \gamma}  C \right) 
\leq \frac{\nu}{t+1 + \gamma},
\end{equation*}
or, equivalently, 
\else
$\delta_{t+1} \leq    \frac{1}{t+\gamma}\big(\big(1- \frac{\StepSizeNume}{ t + \gamma}\ConvParam\big)\nu + \frac{\StepSizeNume^2}{ t+ \gamma}  C \big) $. 
Consequently,  $\delta_{t+1}  \leq \frac{\nu}{t+1 + \gamma}$ holds when 
\fi
\begin{equation}
\left(1- \frac{\StepSizeNume}{ t + \gamma}\ConvParam\right)\nu + \frac{\StepSizeNume^2}{ t+ \gamma}  C \leq \frac{t+ \gamma}{t+1 + \gamma}\nu.
\label{eqn:ConvProof3}
\end{equation}
By setting   ${\nu} \geq \frac{\StepSizeNume^2 C}{ \StepSizeNume \ConvParam- 1}$, the left hand side of \eqref{eqn:ConvProof3} satisfies
\ifFullVersion
\begin{align}
&\left(1\! - \! \frac{\StepSizeNume}{ t \! + \! \gamma}\ConvParam\right)\nu \! + \! \frac{\StepSizeNume^2}{ t\! + \! \gamma} C
= \frac{ t\! - \!1\! + \! \gamma}{t+ \gamma}{\nu} \!+ \! \left( \frac{\StepSizeNume^2 }{t+\gamma}C\! - \! \frac{\StepSizeNume\ConvParam\! - \!1}{t\! + \!\gamma} \nu\right)  \notag \\
& \quad = \frac{ t\! - \!1\! + \! \gamma}{t\! + \! \gamma}{\nu} \!+ \! \frac{\StepSizeNume^2 C \! - \! \left( {\StepSizeNume\ConvParam\! - \!1}\right) \nu }{{t\! + \!\gamma}}  
\!\stackrel{(a)}{\leq}\!\frac{ t\! - \!1\! + \! \gamma}{t\! + \! \gamma}{\nu},
\label{eqn:ConvProof4}
\end{align}
where $(a)$ holds since ${\nu} \geq \frac{\StepSizeNume^2  C}{\StepSizeNume \ConvParam- 1}$. As the right hand side of \eqref{eqn:ConvProof4}
\else
$\big(1\! - \! \frac{\StepSizeNume}{ t \! + \! \gamma}\ConvParam\big)\nu \! + \! \frac{\StepSizeNume^2}{ t\! + \! \gamma} C
  {\leq} \frac{ t\! - \!1\! + \! \gamma}{t\! + \! \gamma}{\nu}$
  since ${\nu} \geq \frac{\StepSizeNume^2  C}{\StepSizeNume \ConvParam- 1}$. As this bound  
\fi  is not larger than that of \eqref{eqn:ConvProof3}, it follows that \eqref{eqn:ConvProof3} holds for the current setting, proving that  $\delta_{t+1}  \leq \frac{\nu}{t+1 + \gamma}$. 
Finally, the smoothness of the objective implies that
\eqspace
\begin{equation}
\E\{\Objective(\myWeights_{t})  \} -  \Objective(\myWeights\Opt) \leq \frac{\SmoothParam}{2}\delta_{t}\leq \frac{\SmoothParam \nu}{2(t+ \gamma)},
\label{eqn:ConvProof5}
\eqspace
\end{equation}
which, in light of the above setting, holds for $\nu = \max \big(\frac{\StepSizeNume^2   C}{\StepSizeNume \ConvParam-1 }, \gamma \delta_{0} \big)$, $\gamma \ge \max(\SGDIter, 4\StepSizeNume \SmoothParam)$, and $\StepSizeNume > 0$. In particular, setting $\StepSizeNume =   \frac{2}{\ConvParam}$ results in  $\gamma =  \max(\SGDIter, 8\SmoothParam/  \ConvParam)$, $\nu = \max \big(\frac{4 C}{ \ConvParam^2 }, \gamma \delta_{0} \big)$ and
\ifFullVersion
\begin{equation}
    \E[\Objective(\myWeights_{t})  ] -  \Objective(\myWeights\Opt)\leq \frac{2\SmoothParam \max \big(4 C,\ConvParam^2 \gamma \delta_{0} \big)  }{\ConvParam^2(t+ \gamma)},
\end{equation} 
\else
$\E[\Objective(\myWeights_{t})  ] -  \Objective(\myWeights\Opt)\leq \frac{2\SmoothParam \max \big(4 C,\ConvParam^2 \gamma \delta_{0} \big)  }{\ConvParam^2(t+ \gamma)}$, 
\fi
thus concluding the proof of Theorem \ref{th: theorem 2}.
\qed
	\vspace{-0.2cm}
	\subsection{Proof of Theorem \ref{th: theorem 3}}
	\label{app:Proof3}
	
    First, as done in Appendix \ref{app:Proof1}, we the virtual sequence $\{\thetabar_t\}$, which here is given by\eqspace
    \begin{equation}
        \label{eq: thetabarUnderFading}
        \!\!\thetabar_{t\!+\!1}\!=\! 
    \begin{cases}
         \frac{1}{N}\sum_{n=1}^N\thetabf^n_{t}, &  t\!+\!1 \notin \mySet{H}, \\
         \frac{1}{K}\sum\limits_{n\in \mySet{K}_t} \thetabf^n_{t} \!+\! \frac{N}{K h_{min}} \myVec{w}_t, & t\!+\!1 \in \mySet{H}.\\
    \end{cases}\eqspace
    \end{equation}
	Let $\bar{\myVec{v}}_t\triangleq \frac{1}{N}\sum_{n=1}^N \thetabf_t^n +\frac{N}{K h_{min}}\myVec{w}_t \mathds{1}_{t \in \mySet{H}} $ be the virtual sequence of the averaged model over all users. Therefore $\bar{\myVec{v}}_t = \thetabar_t$ when $t\notin \mySet{H}$. Under this notation, Theorem \ref{th: theorem 2} characterizes the convergence of $ \E[F(\bar{\myVec{v}}_t)] - F(\tstar)$.   We use the following lemmas, proved in \cite[Appendix B.4]{li2019convergence}.
	\begin{lemma} 
	    \label{lem: lemma5}
	    Under assumption \ref{itm:As4} $\thetabar_t$ is an unbiased estimation of $\bar{\myVec{v}}_t$, i.e. $\E_{\mySet{K}_t} [\thetabar_t] = \bar{\myVec{v}}_t$.
    \end{lemma}
    \begin{lemma} \label{lem: lemma6}
	    The expected difference between $\thetabar_t$ and $\bar{\myVec{v}}_t $ is bounded by:
\ifFullVersion
	    \begin{equation}
	        \label{theta and V difference bound}
	        \E_{\mySet{K}_t}[\|\bar{\myVec{v}}_t - \thetabar_t\|^2] \leq \frac{4(N-K)}{(N-1)K}\eta_t^2 H^2 G^2.
	    \end{equation}
\else
$\E_{\mySet{K}_t}[\|\bar{\myVec{v}}_t - \thetabar_t\|^2] \leq \frac{4(N-K)}{(N-1)K}\eta_t^2 H^2 G^2$.
\fi
	\end{lemma}
	We next use these lemmas to prove the theorem, as
\ifFullVersion
	\begin{align}
	    \|\thetabar_{t+1} - \tstar\|^2 &= \|\thetabar_{t+1}- \bar{\myVec{v}}_{t+1} + \bar{\myVec{v}}_{t+1} -\tstar\|^2 \notag \\
	    & = \underbrace{\|\thetabar_{t+1} - \bar{\myVec{v}}_{t+1}\|^2 } _{ A_1} +\underbrace{\|\bar{\myVec{v}}_{t+1} - \tstar\|^2}_{A_2} \notag\\ 
	    & +\underbrace{ 2\left< \thetabar_{t+1} - \bar{\myVec{v}}_{t+1}, \bar{\myVec{v}}_{t+1} - \tstar \right>}_{A_3}
	    \label{eq: thete to split with v}
	\end{align}
	\else
	$ \|\thetabar_{t+1} - \tstar\|^2  = A_1 + A_2 + A_3$ with $A_1 \triangleq \|\thetabar_{t+1} - \bar{\myVec{v}}_{t+1}\|^2$, $A_2 \triangleq \|\bar{\myVec{v}}_{t+1} - \tstar\|^2$, and $A_3 \triangleq 2\left< \thetabar_{t+1} - \bar{\myVec{v}}_{t+1}, \bar{\myVec{v}}_{t+1} - \tstar \right>$.
	\fi
	The term $E_{\mySet{K}_t}[A_3] = 0$ since $\thetabar_t$ is unbiased by Lemma \ref{lem: lemma5}.
	Further, using Lemma \ref{lem: lemma6}, Theorem \ref{th: theorem 2}, and the equivalent global model in \eqref{eq: received signal normalized fading} to bound $A_1$ and $A_2$ respectively:
	\begin{equation}
	     \E [ \|\thetabar_{t+1} - \tstar\|^2] \leq (1- \eta_t \ConvParam)\E[\|\thetabar_{t} - \tstar\|^2] + \eta_t^2(\tilde{C}+D)
	    \label{eq: theorem 3 mid bound}
	\end{equation}
	where $D = \frac{4(N-K)}{K(N-1)} H^2 G^2$.
	Notice the difference between equations \eqref{eqn:RecursiveRel} and \eqref{eq: theorem 3 mid bound} is in the additional constant $D$, and the scaling of the noise-to-signal ratio in $\tilde{C}$ compared to  $C$ in Theorem \ref{th: theorem 2}. The same arguments used in proving Theorem \ref{th: theorem 2} can now be applied to \eqref{eq: theorem 3 mid bound} to prove the theorem.
	\qed
\end{appendix} 

	\bibliographystyle{IEEEtran}
	\bibliography{IEEEabrv,refs}

\begin{thebibliography}{10}
\providecommand{\url}[1]{#1}
\csname url@samestyle\endcsname
\providecommand{\newblock}{\relax}
\providecommand{\bibinfo}[2]{#2}
\providecommand{\BIBentrySTDinterwordspacing}{\spaceskip=0pt\relax}
\providecommand{\BIBentryALTinterwordstretchfactor}{4}
\providecommand{\BIBentryALTinterwordspacing}{\spaceskip=\fontdimen2\font plus
\BIBentryALTinterwordstretchfactor\fontdimen3\font minus
  \fontdimen4\font\relax}
\providecommand{\BIBforeignlanguage}[2]{{%
\expandafter\ifx\csname l@#1\endcsname\relax
\typeout{** WARNING: IEEEtran.bst: No hyphenation pattern has been}%
\typeout{** loaded for the language `#1'. Using the pattern for}%
\typeout{** the default language instead.}%
\else
\language=\csname l@#1\endcsname
\fi
#2}}
\providecommand{\BIBdecl}{\relax}
\BIBdecl

\bibitem{sery2020cotaf2}
T.~Sery, N.~Shlezinger, K.~Cohen, and Y.~C. Eldar, ``{COTAF}: Convergent
  over-the-air federated learning,'' in \emph{IEEE GLOBECOM}, 2020.

\bibitem{lecun2015deep}
Y.~LeCun, Y.~Bengio, and G.~Hinton, ``Deep learning,'' \emph{Nature}, vol. 521,
  no. 7553, p. 436, 2015.

\bibitem{chen2019deep}
J.~Chen and X.~Ran, ``Deep learning with edge computing: A review,''
  \emph{Proc. {IEEE}}, vol. 107, no.~8, pp. 1655--1674, 2019.

\bibitem{mcmahan2016communication}
H.~B. McMahan, E.~Moore, D.~Ramage, and S.~Hampson, ``Communication-efficient
  learning of deep networks from decentralized data,'' \emph{arXiv preprint
  arXiv:1602.05629}, 2016.

\bibitem{kairouz2019advances}
P.~Kairouz \emph{et~al.}, ``Advances and open problems in federated learning,''
  \emph{arXiv preprint arXiv:1912.04977}, 2019.

\bibitem{smith2017federated}
V.~Smith, C.-K. Chiang, M.~Sanjabi, and A.~S. Talwalkar, ``Federated multi-task
  learning,'' in \emph{Proc. NeurIPS}, 2017, pp. 4424--4434.

\bibitem{shlezinger2020clustered}
N.~Shlezinger, S.~Rini, and Y.~C. Eldar, ``The communication-aware clustered
  federated learning problem,'' in \emph{Proc. IEEE ISIT}, 2020.

\bibitem{li2019federated}
T.~Li, A.~K. Sahu, A.~Talwalkar, and V.~Smith, ``Federated learning:
  Challenges, methods, and future directions,'' \emph{{IEEE} Signal Process.
  Mag.}, vol.~37, no.~3, pp. 50--60, 2020.

\bibitem{speedtest2019}
\BIBentryALTinterwordspacing
speedtest.net, ``Speedtest united states market report,'' 2019. [Online].
  Available: \url{http://www.speedtest.net/reports/united-states/}
\BIBentrySTDinterwordspacing

\bibitem{chen2019joint}
M.~Chen, Z.~Yang, W.~Saad, C.~Yin, H.~V. Poor, and S.~Cui, ``A joint learning
  and communications framework for federated learning over wireless networks,''
  \emph{arXiv preprint arXiv:1909.07972}, 2019.

\bibitem{li2019convergence}
X.~Li, K.~Huang, W.~Yang, S.~Wang, and Z.~Zhang, ``On the convergence of fedavg
  on non-iid data,'' \emph{arXiv preprint arXiv:1907.02189}.

\bibitem{alistarh2017qsgd}
D.~Alistarh, D.~Grubic, J.~Li, R.~Tomioka, and M.~Vojnovic, ``{QSGD}:
  Communication-efficient {SGD} via gradient quantization and encoding,'' in
  \emph{Proc. NeurIPS}, 2017, pp. 1709--1720.

\bibitem{shlezinger2020uveqfed}
N.~Shlezinger, M.~Chen, Y.~C. Eldar, H.~V. Poor, and S.~Cui, ``{UVeQFed}:
  Universal vector quantization for federated learning,'' \emph{arXiv preprint
  arXiv:2006.03262}, 2020.

\bibitem{aji2017sparse}
A.~F. Aji and K.~Heafield, ``Sparse communication for distributed gradient
  descent,'' \emph{arXiv preprint arXiv:1704.05021}, 2017.

\bibitem{alistarh2018convergence}
D.~Alistarh, T.~Hoefler, M.~Johansson, N.~Konstantinov, S.~Khirirat, and
  C.~Renggli, ``The convergence of sparsified gradient methods,'' in
  \emph{Proc. NeurIPS}, 2018, pp. 5973--5983.

\bibitem{goldsmith2005wireless}
A.~Goldsmith, \emph{Wireless communications}.\hskip 1em plus 0.5em minus
  0.4em\relax Cambridge Press, 2005.

\bibitem{amiri2020machine}
M.~M. Amiri and D.~G{\"u}nd{\"u}z, ``Machine learning at the wireless edge:
  Distributed stochastic gradient descent over-the-air,'' \emph{{IEEE} Trans.
  Signal Process.}, vol.~68, pp. 2155--2169, 2020.

\bibitem{amiri2019federated}
------, ``Federated learning over wireless fading channels,'' \emph{{IEEE}
  Trans. Wireless Commun.}, vol.~19, no.~5, pp. 3546--3557, 2020.

\bibitem{sery2019analog}
T.~{Sery} and K.~{Cohen}, ``On analog gradient descent learning over multiple
  access fading channels,'' \emph{{IEEE} Trans. Signal Process.}, 2020.

\bibitem{yang2020federated}
K.~Yang, T.~Jiang, Y.~Shi, and Z.~Ding, ``Federated learning via over-the-air
  computation,'' \emph{{IEEE} Trans. Wireless Commun.}, vol.~19, no.~3, pp.
  2022--2035, 2020.

\bibitem{guo2020analog}
H.~Guo, A.~Liu, and V.~K. Lau, ``Analog gradient aggregation for federated
  learning over wireless networks: Customized design and convergence
  analysis,'' \emph{{IEEE} Internet Things J.}, 2020.

\bibitem{abari2016over}
O.~Abari, H.~Rahul, and D.~Katabi, ``Over-the-air function computation in
  sensor networks,'' \emph{arXiv preprint arXiv:1612.02307}, 2016.

\bibitem{mergen2006type}
G.~Mergen and L.~Tong, ``Type based estimation over multiaccess channels,''
  \emph{{IEEE} Trans. Signal Process.}, vol.~54, no.~2, pp. 613--626, 2006.

\bibitem{Mergen_Asymptotic_2007}
G.~Mergen, V.~Naware, and L.~Tong, ``Asymptotic detection performance of
  type-based multiple access over multiaccess fading channels,'' \emph{{IEEE}
  Trans. Signal Process.}, vol.~55, no.~3, pp. 1081 --1092, Mar. 2007.

\bibitem{Liu_Type_2007}
K.~Liu and A.~Sayeed, ``Type-based decentralized detection in wireless sensor
  networks,'' \emph{{IEEE} Trans. Signal Process.}, vol.~55, no.~5, pp. 1899
  --1910, May 2007.

\bibitem{Marano_Likelihood_2007}
S.~Marano, V.~Matta, T.~Lang, and P.~Willett, ``{A likelihood-based multiple
  access for estimation in sensor networks},'' \emph{{IEEE} Trans. Signal
  Process.}, vol.~55, no.~11, pp. 5155--5166, Nov. 2007.

\bibitem{anandkumar2007type}
A.~Anandkumar and L.~Tong, ``Type-based random access for distributed detection
  over multiaccess fading channels,'' \emph{{IEEE} Trans. Signal Process.},
  vol.~55, no.~10, pp. 5032--5043, 2007.

\bibitem{cohen2013performance}
K.~Cohen and A.~Leshem, ``Performance analysis of likelihood-based multiple
  access for detection over fading channels,'' \emph{{IEEE} Trans. Inf.
  Theory}, vol.~59, no.~4, pp. 2471--2481, 2013.

\bibitem{nevat2014distributed}
I.~Nevat, G.~W. Peters, and I.~B. Collings, ``Distributed detection in sensor
  networks over fading channels with multiple antennas at the fusion centre,''
  \emph{{IEEE} Trans. Signal Process.}, vol.~62, no.~3, pp. 671--683, 2014.

\bibitem{zhang2016event}
P.~Zhang, I.~Nevat, G.~W. Peters, and L.~Clavier, ``Event detection in sensor
  networks with non-linear amplifiers via mixture series expansion,''
  \emph{{IEEE} Sensors J.}, vol.~16, no.~18, pp. 6939--6946, 2016.

\bibitem{cohen2018spectrum}
K.~Cohen and A.~Leshem, ``Spectrum and energy efficient multiple access for
  detection in wireless sensor networks,'' \emph{{IEEE} Trans. Signal
  Process.}, vol.~66, no.~22, pp. 5988--6001, 2018.

\bibitem{cohen2019time}
K.~Cohen and D.~Malachi, ``A time-varying opportunistic multiple access for
  delay-sensitive inference in wireless sensor networks,'' \emph{IEEE Access},
  vol.~7, pp. 170\,475--170\,487, 2019.

\bibitem{seif2020wireless}
M.~Seif, R.~Tandon, and M.~Li, ``Wireless federated learning with local
  differential privacy,'' \emph{arXiv preprint arXiv:2002.05151}, 2020.

\bibitem{liu2020privacy}
D.~Liu and O.~Simeone, ``Privacy for free: Wireless federated learning via
  uncoded transmission with adaptive power control,'' \emph{arXiv preprint
  arXiv:2006.05459}, 2020.

\bibitem{cesa2011online}
N.~Cesa-Bianchi, S.~Shalev-Shwartz, and O.~Shamir, ``Online learning of noisy
  data,'' \emph{{IEEE} Trans. Inf. Theory}, vol.~57, no.~12, pp. 7907--7931,
  2011.

\bibitem{stich2018local}
S.~U. Stich, ``Local {SGD} converges fast and communicates little,''
  \emph{arXiv preprint arXiv:1805.09767}, 2018.

\bibitem{Stich2018sparsified}
S.~U. Stich, J.-B. Cordonnier, and M.~Jaggi, ``Sparsified {SGD} with memory,''
  in \emph{Proc. NeurIPS}, 2018, pp. 4447--4458.

\bibitem{Bertin-Mahieux2011}
T.~Bertin-Mahieux, D.~P. Ellis, B.~Whitman, and P.~Lamere, ``The million song
  dataset,'' in \emph{Proc. ISMIR}, 2011.

\bibitem{Guozhong1995NoiseBackprop}
G.~An, ``The effects of adding noise during backpropagation training on a
  generalization performance,'' \emph{Neural computation}, vol.~8, no.~3, pp.
  643--674, 1996.

\bibitem{chang2020communication}
W.-T. Chang and R.~Tandon, ``Communication efficient federated learning over
  multiple access channels,'' \emph{arXiv preprint arXiv:2001.08737}, 2020.

\bibitem{neelakantan2015adding}
A.~Neelakantan, L.~Vilnis, Q.~V. Le, I.~Sutskever, L.~Kaiser, K.~Kurach, and
  J.~Martens, ``Adding gradient noise improves learning for very deep
  networks,'' \emph{arXiv preprint arXiv:1511.06807}, 2015.

\bibitem{Chen2017checkpoint}
H.~Chen, S.~Lundberg, and S.-I. Lee, ``Checkpoint ensembles: Ensemble methods
  from a single training process,'' \emph{arXiv preprint arXiv:1710.03282},
  2017.

\bibitem{cohen2010time}
K.~Cohen and A.~Leshem, ``A time-varying opportunistic approach to lifetime
  maximization of wireless sensor networks,'' \emph{{IEEE} Trans. Signal
  Process.}, vol.~58, no.~10, pp. 5307--5319, 2010.

\bibitem{Lichman:2013}
\BIBentryALTinterwordspacing
M.~Lichman, ``{UCI} machine learning repository,'' 2013. [Online]. Available:
  \url{http://archive.ics.uci.edu/ml}
\BIBentrySTDinterwordspacing

\bibitem{matlab2020}
\BIBentryALTinterwordspacing
{MathWorks Deep Learning Toolbox Team}, ``Deep learning tutorial series,''
  \emph{MATLAB Central File Exchange}, 2020. [Online]. Available:
  \url{https://www.mathworks.com/matlabcentral/fileexchange/62990-deep-learning-tutorial-series}
\BIBentrySTDinterwordspacing

\end{thebibliography}

\end{document}